\documentclass{article}

\usepackage[final,nonatbib]{neurips_2022}

\usepackage[utf8]{inputenc} 
\usepackage[T1]{fontenc}    
\usepackage{hyperref}       
\usepackage{url}            
\usepackage{booktabs}       
\usepackage{amsfonts}       
\usepackage{nicefrac}       
\usepackage{microtype}      
\usepackage{xcolor}         
\usepackage{wrapfig}

\usepackage{graphicx}
\usepackage{subfigure}
\usepackage{array}
\usepackage{enumitem}

\usepackage{amsmath}
\usepackage{amsthm}
\usepackage{amssymb}
\usepackage{mathtools}
\usepackage{color}
\usepackage[ruled,vlined]{algorithm2e}
\usepackage{dsfont}
\usepackage{footnote}
\makesavenoteenv{tabular}
\usepackage{multirow}

\newcommand{\ba}{\left[ \begin{array}}
\newcommand{\ea}{\\ \end{array} \right]}

\def\x{{\boldsymbol{x}}}
\def\y{{\boldsymbol{y}}}

\newcommand{\vg}{{\mathbf{g}}}

\newcommand{\vx}{{\mathbf{x}}}
\newcommand{\vy}{{\mathbf{y}}}

\newcommand{\vS}{{\mathbf{S}}}

\newcommand{\cA}{{\mathcal{A}}}

\newcommand{\cF}{{\mathcal{F}}}

\newcommand{\cN}{{\mathcal{N}}}
\newcommand{\cO}{{\mathcal{O}}}

\newcommand{\cR}{{\mathcal{R}}}

\newcommand{\cW}{{\mathcal{W}}}
\newcommand{\cX}{{\mathcal{X}}}

\newcommand{\bx}{{\boldsymbol{x}}}

\newcommand{\EE}{\mathbb{E}}
\newcommand{\PP}{\mathbb{P}}
\newcommand{\RR}{\mathbb{R}}

\newcommand{\dist}{{\mathrm{dist}}}

\newcommand{\one}{\mathds{1}}
\newcommand{\diag}{{\mathrm{diag}}} 

\newcommand{\vvvert}{{\vert\kern-0.25ex\vert\kern-0.25ex\vert}}

\newcommand{\prog}{\mathrm{prog}}

\newtheorem{theorem}{Theorem}
\newtheorem{definition}{Definition}

\newtheorem{remark}{Remark}
\newtheorem{lemma}{Lemma}
\newtheorem{proposition}{Proposition}

\newcommand{\tablefontsize}{\scriptsize}

\providecommand{\psgd}{PSGD\xspace}
\providecommand{\dsgd}{DSGD\xspace}
\providecommand{\dsgt}{DSGT\xspace}
\providecommand{\dsquare}{D$^2$\xspace}
\providecommand{\ours}{MG-DSGD\xspace}
\providecommand{\detag}{DeTAG\xspace}
\providecommand{\cifar}{CIFAR-10\xspace}
\providecommand{\imagenet}{ImageNet\xspace}

\title{Revisiting Optimal Convergence Rate for Smooth and Non-convex Stochastic Decentralized Optimization}

\author{%
Kun Yuan$^{1,3}$\thanks{Equal Contribution. Corresponding Author: Kun Yuan},~ Xinmeng Huang$^{2*}$, Yiming Chen$^{1,4*}$, Xiaohan Zhang$^2$, Yingya Zhang$^{1}$, Pan Pan$^1$ \vspace{1mm}\\ 
    $^1$DAMO Academy, Alibaba Group\quad $^2$University of Pennsylvania \\
    $^3$Peking  University\quad $^4$MetaCarbon  \vspace{1mm}\\
{\small\texttt{\{kun.yuan, yingya.zyy, panpan.pp\}@alibaba-inc.com}}\\{\small\texttt{xinmengh@sas.upenn.edu\quad yiming@metacarbon.vip \quad  zxiaohan@seas.upenn.edu}} 
}

\begin{document}

\maketitle

\begin{abstract}
Decentralized optimization is effective to save communication in large-scale machine learning. Although numerous algorithms have been proposed with theoretical guarantees and empirical successes, the performance limits in decentralized optimization, especially the influence of network topology and its associated weight matrix on the optimal convergence rate, have not been fully understood. While Lu and Sa \cite{lu2021optimal} have recently provided an optimal rate for non-convex stochastic decentralized optimization with weight matrices defined over linear graphs, the optimal rate with {\em general} weight matrices remains unclear. 

This paper revisits non-convex stochastic decentralized optimization and establishes an optimal convergence rate with general weight matrices. In addition, we also establish the optimal rate when non-convex loss functions further satisfy the Polyak-Lojasiewicz (PL) condition. Following existing lines of analysis in literature cannot achieve these results. Instead, we leverage the Ring-Lattice graph to admit general weight matrices while maintaining the optimal relation between the graph diameter and weight matrix connectivity. Lastly, we develop a new decentralized algorithm to nearly attain the above two optimal rates under additional mild conditions. 
\end{abstract}

\section{Introduction}

\textbf{1.1~Motivation.} Decentralized optimization is an emerging paradigm for large-scale machine learning. By letting each node average with its neighbors, decentralized algorithms save great communication overhead compared to traditional approaches with a central server. While numerous effective decentralized algorithms (see, e.g., \cite{nedic2009distributed,chen2012diffusion,shi2015extra,lian2017can,assran2019stochastic,lin2021quasi,yuan2021decentlam}) have been proposed in literature showing both theoretical guarantees and empirical successes, 
the performance limits in decentralized optimization have not been fully clarified. This paper provides further understandings in optimal convergence rate in decentralized optimization and develops algorithms to achieve these rates. 

\textbf{Weight matrix and connectivity measure.} Assume $n$ computing nodes are connected by some network topology (graph). We associate the topology with a weight matrix $W = [w_{ij}]_{i,j=1}^n \in \RR^{n\times n}$ in which $w_{ij} \in (0,1)$ if node $j$ is connected to node $i$ otherwise $w_{ij} = 0$. We introduce $\beta:= \|W - \frac{1}{n}\mathds{1}_n \mathds{1}_n^T\| \in (0,1)$, where $\mathds{1}_n=(1,\dots,1)^T \in\RR^n$ has all entries being $1$, as the {\em connectivity measure} to gauge how well the network topology is connected. Quantity $\beta \to 0$ (which implies $W \to \frac{1}{n}\mathds{1}_n \mathds{1}_n^T$) indicates a well-connected topology while $\beta \to 1$ (which implies $W \to I$) indicates a badly-connected topology. It is the weight matrix $W$ and its connectivity measure $\beta$ that bring major challenges to convergence analysis in decentralized optimization. 

\textbf{Decentralized optimization.} 
Decentralized approaches are built upon the  partial averaging $x_i^{+} = \sum_{j\in \cN_i}w_{ij} x_j$ (or $x^+ = W x$ in a more compact manner) where $\cN_i$ is the set of neighbors of node $i$ (including node $i$ itself). The sparsity and connectivity of $W$ significantly influence the efficiency and effectiveness of the partial averaging. Decentralized gradient descent \cite{nedic2009distributed,yuan2016convergence}, diffusion \cite{chen2012diffusion,sayed2014adaptive} and dual averaging \cite{duchi2011dual} are well-known decentralized methods. Other advanced algorithms, which can correct the bias caused by heterogeneous data distributions, include decentralized ADMM \cite{mateos2010distributed,shi2014linear}, explicit bias-correction \cite{shi2015extra,yuan2017exact1,li2017decentralized}, gradient tracking \cite{nedic2017achieving,di2016next,qu2018harnessing,xu2015augmented}, and dual acceleration \cite{scaman2017optimal,uribe2020dual}.

In the stochastic regime, decentralize SGD \cite{chen2012diffusion,sayed2014adaptive,lian2017can,koloskova2020unified} and its momentum and adaptive variants \cite{lin2021quasi,yuan2021decentlam,nazari2019dadam} have attracted a lot of attentions. After being extended to directed topologies \cite{assran2019stochastic,ying2021exponential}, time-varying topologies \cite{koloskova2020unified,nedic2014distributed,wang2019matcha}, asynchronous settings \cite{lian2018asynchronous}, and data-heterogeneous scenarios \cite{tang2018d,xin2020improved,lu2019gnsd,alghunaim2021unified,koloskova2021improved}, decentralize SGD has much more applicable scenarios with significantly improved performance. Decentralized algorithms are also closely related to federated learning methods \cite{mcmahan2017communication,stich2019local,yu2019linear,karimireddy2020scaffold,li2019convergence} when they admit weight matrices alternating between $\frac{1}{n}\mathds{1}_n\mathds{1}_n^T$ and the identity matrix. 

\textbf{Prior understandings in theoretical limits.} A series of pioneering works have shed lights on the theoretical limits in decentralized optimization. These works establish the optimal convergence rate for convex or non-stochastic decentralized optimization \cite{scaman2017optimal,scaman2018optimal,sun2019distributed,kovalev2021lower}, and propose decentralized algorithms that (nearly) match these optimal bounds \cite{scaman2017optimal,scaman2018optimal,uribe2020dual,sun2019distributed,kovalev2020optimal,kovalev2021lower}. However, there are few studies on the theoretical limits in non-convex stochastic decentralized optimization, which is quite  common  
in large-scale deep learning, even when the costs are smooth. 

\begin{table}[t]
\centering 
\caption{\small Rate comparison between different algorithms in smooth and non-convex stochastic decentralized optimization. Parameter $n$ denotes the number of all computing nodes, $\beta\in [0,1)$ denotes the connectivity measure of the weight matrix, $\sigma^2$ measures the gradient noise, $b^2$ denotes data heterogeneity, and $T$ is the number of iterations. Other constants such as the initialization $f(x^{(0)}) - f^\star$ and smoothness constant $L$ are omitted for clarity. Detailed complexities with constant $L$ are listed in Table \ref{table-non-convex-app} in Appendix \ref{app-detailed-comp}. The definition of transient iteration complexity can be found in Remark \ref{rm-transient-iteration} (the smaller the better). ``LB'' is lower bound while ``UB'' is upper bound. Notation $\tilde{O}(\cdot)$ hides all logarithm factors.}
\begin{tabular}{rllll}
\toprule
                    & \textbf{References}                                                                           & \textbf{Gossip matrix}                     & \textbf{Convergence rate}    & \hspace{-2mm} \textbf{Tran. iters.}\\ \midrule
\multirow{2}{*}{LB} & \cite{lu2021optimal}                                                & $\beta = \cos(\pi/n)$ &         $\Omega\big( \frac{\sigma}{\sqrt{nT}}+ \frac{1}{T (1-\beta)^{\frac{1}{2}}} \big)$              &            $O(\frac{n}{(1-\beta)\sigma^2})$                            \vspace{1mm}   \\
                             & {\color{blue}Theorem \ref{thm-lower-bound-nc}}                                                                           & {\color{blue}${\beta \in [0, \cos(\pi/n)]}$}       &          {\color{blue}$\Omega\big( \frac{\sigma}{\sqrt{nT}}+ \frac{1}{T (1-\beta)^{\frac{1}{2}}} \big)$}            &            {\color{blue}$O(\frac{n}{(1-\beta)\sigma^2})$}                                    \\ \midrule
\multirow{5}{*}{UB} & DSGD \cite{koloskova2020unified}                                    & $\beta \in [0, 1)$           &  \hspace{-1.2cm}    $O\big( \frac{\sigma}{\sqrt{nT}} \hspace{-1mm}+\hspace{-1mm} \frac{\sigma^{\frac{2}{3}}}{T^{\frac{2}{3}}(1-\beta)^{\frac{1}{3}}} \hspace{-1mm}+\hspace{-1mm} \frac{b^{\frac{2}{3}}}{T^{\frac{2}{3}}(1-\beta)^{\frac{2}{3}}} \big)$                &                    ${O}\big(\frac{n^3}{(1-\beta)^2\sigma^2}\big)^\dagger$          \vspace{1mm}    \\
                             & D$^2$/ED \cite{tang2018d}                   & $\beta \in [0, 1)$           &        $O\big( \frac{\sigma}{\sqrt{nT}} + \frac{1}{T (1-\beta)^3} \big)$                &       ${O}(\frac{n}{(1-\beta)^6\sigma^2})$                                          \vspace{1mm}\\
                             & DSGT \cite{koloskova2021improved} & $\beta \in [0, 1)$           & $\tilde{O}\big( \frac{\sigma}{\sqrt{nT}} + \frac{\sigma^{\frac{2}{3}}}{T^{\frac{2}{3}}(1-\beta)^{\frac{1}{3}}}  \big)$  &        $\tilde{O}(\frac{n^3}{(1-\beta)^2\sigma^2})$                                     \vspace{1mm}     \\
                             & DeTAG \cite{lu2021optimal}                                          & $\beta \in [0, 1)$           & $\tilde{O}\big( \frac{\sigma}{\sqrt{nT}}+ \frac{1}{T (1-\beta)^{\frac{1}{2}}} \big)$ &                  $\tilde{O}(\frac{n}{(1-\beta)\sigma^2})$                           \vspace{1mm}   \\
                             & {\color{blue}MG-DSGD}                                                                              & {\color{blue}$\beta \in [0, 1)$}           & {\color{blue}$\tilde{O}\big( \frac{\sigma}{\sqrt{nT}}+ \frac{1}{T (1-\beta)^{\frac{1}{2}}}\big)$} &                       {\color{blue}$\tilde{O}(\frac{n}{(1-\beta)\sigma^2})$}                          \\ \bottomrule
\multicolumn{5}{l}{$^\dagger$\,\footnotesize{The complete complexity is ${O}(\frac{n^3}{(1-\beta)^2\min\{\sigma^2,(1-\beta)^2 b^2\}})$, which reduces to ${O}(\frac{n^3}{(1-\beta)^2\sigma^2})$ for small $\sigma^2$.}}
\end{tabular}
\vspace{-6mm}
\label{table-non-convex}
\end{table}

A recent novel work \cite{lu2021optimal} establishes the optimal complexity for smooth and non-convex stochastic decentralized optimization. While this optimal bound is inspiring, it is valid for a  restrictive family of weight matrix $W$ associated with the linear graph whose connectivity measure $\beta = \cos(\pi/n)$\footnote{See 
Corollary 1 in the latest arXiv version of \cite{lu2021optimal} uploaded in Jan. 2022: arXiv:2006.08085v4.}.
However, the linear graph is just one of the commonly used topologies. A more difficult but fundamental question still remains open: 

{\em What is the optimal convergence rate of smooth and non-convex stochastic decentralized optimization with a general weight matrix (that does not necessarily satisfy $\beta = \cos(\pi/n)$)?}

It is not known whether the lower bound in \cite{lu2021optimal} still holds for other commonly-used weight matrices resulted from grids, toruses, hypercubes, etc. whose connectivity measure $\beta \neq \cos(\pi/n)$. 

\textbf{Challenges.} Following existing lines of analysis, to our knowledge, cannot provide an assuring answer. It is known that the communication complexity in decentralized optimization is typically proportional to the diameter $D$ of the network topology \cite{scaman2017optimal}. Establishing the relation between $D$ and $\beta$ is key to clarifying the influence of the weight matrix on optimal convergence rate.  Existing works \cite{scaman2017optimal,scaman2018optimal,lu2021optimal} utilize the linear graph to establish the optimal relation as $D = \Theta(1/\sqrt{1-\beta})$. However, weight matrices associated with linear graph are quite {\em limited}. Given 
a fixed network size $n$ (which is a prerequisite when discussing complexities in distributed stochastic optimization \cite{lu2021optimal,Horvath2019StochasticDL,Koloskova2020AUT,fang2018spider}), they have to satisfy $\beta = \cos(\pi/n)$. Weight matrices with other $\beta$ cannot be implied from the linear graph. To examine the lower bounds with a more general weight matrix $W$, \cite{sun2019distributed} proposes a wider class of networks named linear-star graph, which, however, leads to {\em suboptimal} relation between $D$ and $\beta$. As a result, a network topology that admits general weight matrices $W$ while maintaining the optimal relation between $D$ and $\beta$ is critical to clarify the above fundamental question.

\textbf{1.2 Main results.}  This paper successfully identifies the desired network topology which facilitates the theoretical limits over the general weight matrices. Our contributions are summarized as follows.

\begin{itemize}[leftmargin=1em]
    \item We discover that ring-lattice graph maintains the optimal relation between $D$ and $\beta$ for a much broader class of weight matrices. Given any fixed $n$ and  $\beta \in [0, \cos(\pi/n)]$, we can always construct a weight matrix $W\in \RR^{n\times n}$ associated with some ring-lattice graph, whose connectivity measure is $\beta$, that satisfies $D = \Theta(1/\sqrt{1-\beta})$. Since $\cos(\pi/n)$ approaches to $1$ when $n$ is large, the ring-lattice graph admits far broader weight matrices than linear graph used in \cite{scaman2017optimal,scaman2018optimal,lu2021optimal}. 
    
    \item With the above results, we provide the optimal convergence rate for smooth and non-convex stochastic decentralized optimization that holds for any weight matrix with $\beta \in [0, \cos(\pi/n)]$. 

    In contrast, the optimal rate in \cite{lu2021optimal} only applies to weight matrix satisfying $\beta  = \cos(\pi/n)$. We also provide the optimal convergence rate when the non-convex cost functions further satisfy the Polyak-Lojasiewicz (PL) condition \cite{karimi2016linear} for any weight matrix with $\beta \in [0, \cos(\pi/n)]$.  
    
    \item We prove that the above two optimal complexities can be {nearly-attained} (under mild assumptions and up to {the condition number $L/\mu$ as well as} other logarithmic factors) by simply integrating {multiple gossip communication \cite{liu2011accelerated,rogozin2021towards}} and {gradient accumulation \cite{scaman2017optimal,Rogozin2021AnAM,lu2021optimal}} to the vanilla decentralized SGD (D-SGD) algorithm. Compared to DeTAG \cite{lu2021optimal}, the extended D-SGD avoids the gradient tracking steps and saves half of the communication overheads. {Compared to DAGD \cite{Rogozin2021AnAM}, the extended D-SGD achieves the near-optimality using constant mini-batch sizes and has theoretical guarantees in the non-convex (as well as the PL) scenario.}
\end{itemize}

All established results in this paper as well as those of existing state-of-the-art decentralized algorithms are listed in Tables \ref{table-non-convex} and \ref{table-non-convex-PL}. The smoothness constant $L$ and PL constant $\mu$ are
omitted in these tables for clarity. The detailed complexities of these algorithms with constant $L$ and $\mu$ are listed in {Tables \ref{table-non-convex-app} and \ref{table-non-convex-PL-app} in Appendix \ref{app-detailed-comp}}.

\begin{table}[t]
\centering 
\caption{\small Rate comparison between different algorithms in smooth and non-convex stochastic decentralized optimization under the PL condition. Constants such as the initialization $f(x^{(0)}) - f^\star$ and smoothness constant $L$ are omitted for clarity. Detailed complexities with constants $L$ and $\mu$ are listed in Table \ref{table-non-convex-app} in Appendix \ref{app-detailed-comp}.}
\begin{tabular}{rllll}
\toprule
& \textbf{References}                                                                           & \textbf{Gossip matrix}                     & \textbf{Convergence rate}    & \hspace{-2mm} \textbf{Tran. iters. complexity}\\ \midrule
LB
& {\color{blue}Theorem \ref{thm-lower-bound-pl}}                                                                           & {\color{blue}$\beta \in [0, \cos(\pi/n)]$}       &          {\color{blue}$\Omega\big( \frac{\sigma^2}{nT}+ e^{-T(1-\beta)^{\frac{1}{2}}} \big)$}            &            {\color{blue}$\tilde{O}(\frac{1}{(1-\beta)^{1/2}})$}                                    \\ \midrule
\multirow{5}{*}{UB} & DSGD \cite{koloskova2020unified}$^\dagger$                                    & $\beta \in [0, 1)$           &  \hspace{-1.2cm}    $\tilde{O}\big( \frac{\sigma^2}{nT} \hspace{-0.5mm}+\hspace{-0.5mm} \frac{\sigma^2}{T^2(1-\beta)} \hspace{-0.5mm}+\hspace{-0.5mm} \frac{b^2}{T^2 (1-\beta)^{2}} \big)$                &                    $\tilde{O}(\frac{n}{(1-\beta)\min\{1,(1-\beta) b^2\}})$           \vspace{1mm}    \\
& {DAGD\cite{Rogozin2021AnAM}$^\dagger$$^\ddag$}                 & {$\beta \in [0, 1)$}       &        \hspace{-1.2cm}  {$\tilde{O}\big( \frac{\sigma^2}{nT}+  e^{-T(1-\beta)^{\frac{1}{2}}} \big)$}            &        {$\tilde{O}(\frac{1}{(1-\beta)^{1/2}})$}                                 \vspace{1mm}    \\
 & D$^2$/ED \cite{alghunaim2021unified,yuan2021removing}                   & $\beta \in [0, 1)$           &      \hspace{-1.2cm}    $\tilde{O}\big( \frac{\sigma^2}{nT} \hspace{-0.5mm}+\hspace{-0.5mm} \frac{\sigma^2}{T^2(1-\beta)} \hspace{-0.5mm}+\hspace{-0.5mm} e^{-T(1-\beta)} \big)$               &       $\tilde{O}(\frac{n}{1-\beta})$                                          \vspace{1mm}\\
 & DSGT \cite{alghunaim2021unified} & $\beta \in [0, 1)$           &  \hspace{-1.2cm}    $\tilde{O}\big( \frac{\sigma^2}{nT} \hspace{-0.5mm}+\hspace{-0.5mm} \frac{\sigma^2}{T^2(1-\beta)} \hspace{-0.5mm}+\hspace{-0.5mm} e^{-T(1-\beta)} \big)$       &        $\tilde{O}(\frac{n}{1-\beta})$                                              \vspace{1mm}     \\
 & DSGT \cite{xin2020improved}                                          & $\beta \in [0, 1)$    &       \hspace{-1.2cm}    $\tilde{O}\big( \frac{\sigma^2}{nT} \hspace{-0.5mm}+\hspace{-0.5mm} \frac{\sigma^2}{T^2(1-\beta)^3} \hspace{-0.5mm}+\hspace{-0.5mm} e^{-T(1-\beta)} \big)$  &                  $\tilde{O}(\frac{n}{(1-\beta)^3})$                           \vspace{1mm}   \\
 & {\color{blue}MG-DSGD}                                                                              & {\color{blue}$\beta \in [0, 1)$}           & {\color{blue}$\tilde{O}\big( \frac{\sigma^2}{nT}+ e^{-T(1-\beta)^{\frac{1}{2}}} \big)$}   &                       {\color{blue}$\tilde{O}(\frac{1}{(1-\beta)^{1/2}})$}                       \\ \bottomrule  
\multicolumn{5}{l}{$^\dagger$\,\footnotesize{This rate is derived under the strongly-convex assumption.}}\\
\multicolumn{5}{l}{$^\ddag$\,\footnotesize{This rate is achieved by utilizing increasing (non-constant) mini-batch sizes.}}
\end{tabular}
\label{table-non-convex-PL}
\end{table}

\textbf{1.3 Other related works.} Some other related works are discussed as follows.

\textbf{Lower bounds and optimal algorithms.} The lower bounds and optimal algorithms for stochastic and convex optimization have been intensively studied in \cite{lan2012optimal,lan2018optimal,rakhlin2012making,agarwal2009information,diakonikolas2019lower}. References \cite{carmon2020lower,carmon2021lower} established the lower bounds to find stationary solutions for  non-convex optimization, and \cite{fang2018spider,allen2018make} proposed algorithms with near-optimal convergence rate for non-convex finite-sum minimization. In deterministic decentralized optimization, the optimal convergence rate was established for strongly convex optimization \cite{scaman2017optimal}, non-smooth convex optimization \cite{scaman2018optimal}, and non-convex  optimization \cite{sun2019distributed}. Furthermore, the lower bounds and optimal algorithms over time-varying topologies are studied in \cite{kovalev2020optimal,Rogozin2021AnAM,rogozin2021towards}. 
In stochastic decentralized optimization, works \cite{yuan2021removing} and \cite{koloskova2020unified} established the lower bounds for strongly-convex decentralized SGD. 
{\cite{lu2021optimal} proves the optimal complexities for non-convex problems when weight matrices satisfy $\beta =\cos(\pi/n)$. }

\textbf{Ring-lattice Graph.} The ring-lattice graph was used in \cite{watts1999small,hayes2000computing,Chen2022OnWG} as a basis model to generate the small world graph. However, its relation between graph diameter and connectivity measure is not clarified to our knowledge. A related work \cite{wu2012robustness} examines the robustness of the ring-lattice graph, but does not discuss its connectivity measure. Furthermore, ring-lattice has never been utilized to establish the influence of the weight matrix on the optimal rate in decentralized optimization before.

\section{Problem setup}
\label{sec-setup}
\textbf{Problem setup.} Consider the following problem with a network of $n$ computing nodes:
\begin{align}\label{dist-opt}
\setlength{\abovedisplayskip}{3pt}\setlength{\belowdisplayskip}{3pt}
{\color{black} \min_{x \in \mathbb{R}^d}\  f(x)=\frac{1}{n}\sum_{i=1}^n f_i(x) \quad \mbox{where} \quad f_i(x): = \mathbb{E}_{\xi_i \sim D_i} F(x;\xi_i).}
\end{align}
Function $f_i(x)$ is local to node $i$, and random variable $\xi_i$ denotes the local data that follows distribution $D_i$. Each local data distribution $D_i$ can be different across all nodes. 

\noindent \textbf{Assumptions.} The optimal convergence rate is established under the following assumptions. 

\vspace{-2mm}
\begin{itemize}[leftmargin=1.5em]
    \item \textbf{Function class.} We assume each local loss function is $L$-smooth, i.e., 
    \begin{align}\label{L-smooth}
    \|\nabla f_i(x) - \nabla f_i(y)\| \le L \|x - y\|, \quad \forall x, y \in \RR^n,
    \end{align}
    for some constant $L>0$. We let the class $\cF_L$ denote the set of all functions satisfying \eqref{L-smooth}. We may also assume the global loss  $f(x)=\frac{1}{n}\sum_{i=1}^nf_i(x)$ satisfies the $\mu$-PL condition: 
   \begin{align}\label{PL-cond}
    2\mu(f(x) - f^\star) \le \|\nabla f(x)\|^2, \quad \forall x\in \RR^n,
    \end{align}
    where $f^\star$ is the optimal function value in problem \eqref{dist-opt} and $\mu > 0$ is some constant. We let the function class $\cF_{L,\mu}$ denote the set of all functions satisfying both \eqref{L-smooth} and \eqref{PL-cond}. 
    \item \textbf{Oracle class.} We assume each node $i$ interacts with $f_i(x)$ via the stochastic gradient oracle  $\tilde{g}_i(x,\xi)$ where $\xi$ is a random variable. In addition, we assume oracle $\tilde{g}_i(x,\xi)$ satisfies 
    \begin{align}
        \mathbb{E}_{\xi }[\tilde{g}_i(x,\xi)] = \nabla f_i(x) \quad \mbox{and} \quad 
        \mathbb{E}_{\xi}\|\tilde{g}_i(x,\xi) - \nabla f_i(x)\|^2 \le \sigma^2. 
    \end{align}
    We let $\cO_{\sigma^2}$ denote the class of all such stochastic gradient oracles. 
    
    \item \textbf{Weight matrix class.} We let $\cW_{n,\beta}$ denote the class of all gossip matrices $W\in \RR^{n\times n}$ satisfying 
    \begin{align}\label{eqn:W-assumption}
    W\mathds{1}_n = \mathds{1}_n, \quad \mathds{1}_n^T W = \mathds{1}_n^T, \quad \|W - \frac{1}{n}\mathds{1}_n\mathds{1}_n^T\| = \beta \in [0,1),
    \end{align}
    where $\mathds{1}_n=(1,\dots,1)^T \in\RR^n$ with all entries being $1$.
    If $W$ is positive semi-definite, then connectivity measure $\beta$ is its second largest eigenvalue. 
    
    \item \textbf{Algorithm class.} We consider an algorithm $A$ in which each node $i$ accesses an unknown local function $f_i $ via the stochastic gradient oracle $\tilde{g}_i(x;\xi_i) \in \cO_{\sigma^2}$.  Each node $i$ running algorithm $A$ will maintain a local model copy $x^{(k)}_{i}$ at iteration $k$. 

    We assume $A$ to follow the partial averaging policy, i.e., each node communicates via protocol 
    $$z_i^{(k)} = \sum_{j\in \mathcal{N}_i} w_{ij} y_j^{(k)}, \quad \forall i \in [n]$$ for some $W = [w_{ij}]_{i,j=1}^n \in \cW_{n, \beta}$ where $[n]:=\{1,\cdots,n\}$, and $y$ and $z$ are the input and output variables of the communication protocol. In addition, we assume $A$ to follow the zero-respecting policy \cite{carmon2020lower,carmon2021lower,huang2022lb}. Informally speaking, the zero-respecting policy requires that the number of non-zero entries of local model copy $x_i^{(k)}$ can only be increased by sampling its own stochastic gradient oracle or interacting with the neighboring nodes. The zero-respecting policy is widely followed by all methods in Tables \ref{table-non-convex} and \ref{table-non-convex-PL} and their momentum and adaptive variants. We let $\cA_{W}$ be the set of all algorithms following the partial averaging and zero-respecting policies. 
\end{itemize}

\vspace{-1mm}
\textbf{Lower bound metrics.} With the above classes, we now introduce the measures for  convergence rate's lower bound.  Given loss functions $\{f_i\}_{i=1}^n\subseteq \cF_{L}$ (or $\{f_i\}_{i=1}^n\subseteq \cF_{L,\mu}$), stochastic gradient oracles $\{\tilde{g}_i\}_{i=1}^n \subseteq  \cO_{\sigma^2}$, a weight matrix $W\in\cW_{n,\beta}$, an algorithm $A\in \cA_W$, the number of 
gradient queries (or communications rounds) $T$, we let $\hat{x}_{A, \{f_i\}_{i=1}^n,\{\tilde{g}_i\}_{i=1}^n,W,T}$ (which will be denoted as $\hat{x}$ when there is no ambiguity) be the output of algorithm $A$ with $\hat{x} \in \mathrm{Span}\big(\{\{x_j^{(t)}\}_{j=1}^n\}_{t=1}^T\big)$. The lower bound for smooth and non-convex stochastic decentralized optimization is defined as 
\begin{align}
    \inf_{A\in \cA_W}\sup_{W\in \cW_{n,\beta}} \sup_{\{\tilde{g}_i\}_{i=1}^n\in \cO_{\sigma^2}}\sup_{\{f_i\}_{i=1}^n\subseteq \cF_{L}}\mathbb{E}\|\nabla f(\hat{x}_{A, \{f_i\}_{i=1}^n,\{\tilde{g}_i\}_{i=1}^n,W,T})\|^2.
\end{align}
The lower bound under the PL condition is defined as 
\begin{align}
    \inf_{A\in \cA_W}\sup_{W\in \cW_{n,\beta}} \sup_{\{\tilde{g}_i\}_{i=1}^n\in \cO_{\sigma^2}}\sup_{\{f_i\}_{i=1}^n\subseteq \cF_{L,\mu}} \mathbb{E}[f(\hat{x}_{A, \{f_i\}_{i=1}^n,\{\tilde{g}_i\}_{i=1}^n,W,T}) - f^\star].
\end{align}

\vspace{-1mm}
\textbf{Fixed network size $n$.} This paper considers {\em stochastic} decentralized optimization where {\em the network size $n$ is a fixed constant}. A fixed $n$ is a prerequisite in distributed stochastic optimization which enables distributed algorithms to achieve the linear speedup in convergence rate $O(\sigma/\sqrt{nT})$, see the algorithms listed in Tables \ref{table-non-convex} and \ref{table-non-convex-PL}. In decentralized deterministic optimization, however, size $n$ does not appear in the convergence rate. Thus, it does not need to be specified 
and can be varied freely when establishing the lower bounds \cite{scaman2017optimal,scaman2018optimal}.

\begin{figure}[t]
\centering
\includegraphics[scale=0.18]{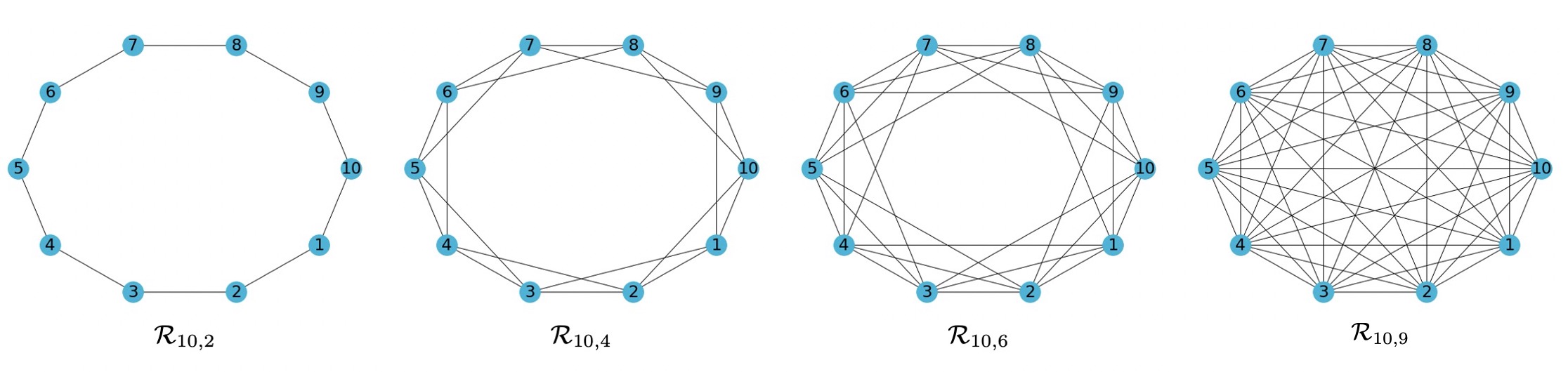}
\caption{\small An illustration of the ring-lattice graph with size $10$ and degrees $2,4,6,9$. It is observed that $\cR_{10,2}$ is a ring graph while $\cR_{10,9}$ is a complete graph.}
\label{fig:ring-lattice}
\vspace{-3mm}
\end{figure}

\vspace{-1mm}
\section{Ring-Lattice graph}
\vspace{-1mm}
The main challenge in analyzing theoretical limits in decentralized optimization lies in clarifying the influence of network topology and its associated weight matrix. It is intuitive that the graph diameter $D$ affects the convergence rate. In the worst case, one node needs to take at least $D$ communication steps to transmit its own gradient information to the other.  If the relation between $D$ and $\beta$ is established, the influence of the weight matrix on lower bounds will become clear. However, existing results either establishes suboptimal relation between $D$ and $\beta$ \cite{sun2019distributed} or admits restrictive weight matrices  \cite{scaman2017optimal,scaman2018optimal,lu2021optimal} given fixed $n$. Now we show that the ring-lattice graph maintains the optimal relation between $D$ and $\beta$ for a broad class of weight matrices. All proof details are in Appendix \ref{app-ring-lattice}. 

\begin{definition}[\sc Ring-Lattice Graph \cite{watts1998collective}]\label{def:rl}
    Given any positive integers $n,k$ such that $k$ is even and {$2\leq k < n-1$}, the $k$-regular ring-lattice graph over nodes $\{1,\cdots, n\}$, denoted by $\cR_{n,k}$, is an undirected graph in which each node $i \in [n]$ is connected with the set of nodes $\{(i+\ell) \;\mathrm{mod}\;n: 1\leq |\ell|\leq k/2\}$ where  $(i+\ell) \;\mathrm{mod}\;n$ is defined as 
    \[
    (i+\ell) \;\mathrm{mod}\;n =\begin{cases}
    i+\ell &\text{if $1\leq  i+\ell\leq n$};\\
    i+\ell+n &\text{if $i+\ell<1$};\\
    i+\ell-n &\text{if $i+\ell> n$}.\\
    \end{cases}
    \]
    Apparently, each node in $\cR_{n,k}$ has degree $k$.
\end{definition}
Ring-lattice is a $k$-regular graph. When $k=2$, ring-lattice will reduce to the ring graph. While Definition \ref{def:rl} excludes the case in which  {$k=n-1$}, the complete graph can be exceptionally viewed as the ring lattice with degree {$k=n-1$}. Therefore, for general $2\le k \le n-1$, $\cR_{n,k}$ can be regarded as an intermediate state between the ring and complete graphs, see the illustration in Fig.~ \ref{fig:ring-lattice}. The following lemma clarifies the diameter of $\cR_{n,k}$. 

\begin{lemma}
\label{lm-diameter}
The diameter of the ring-lattice graph $\cR_{n,k}$ is $D = \Theta(n/k)$. 
\end{lemma} 
\vspace{-2mm}
It is derived from the above lemma that the diameters of $\cR_{n,2}$ and $\cR_{n,n}$ are $\Theta(n)$ and $\Theta(1)$, respectively, which are consistent with the ring and complete graphs. 

The following fundamental theorem establishes the relation between $D$ and connectivity measure $\beta$.

\begin{theorem}\label{thm-D-beta-relation}
Given a fixed $n$ and any $\beta \in [0, \cos(\pi/n)]$, we can always construct a ring-lattice graph $\cR_{n,k}$ so that
\begin{itemize}[leftmargin=3em]
\vspace{-1mm}
\item[(1)] it has an associated weight matrix $W \in \cW_{n,\beta}$, i.e., $W \in \RR^{n\times n}$ and $\|W - \frac{1}{n}\mathds{1}_n\mathds{1}_n^T\| = \beta$,
\vspace{-1mm}
\item[(2)] its diameter $D$ satisfies $D = \Theta(1/\sqrt{1-\beta})$.
\end{itemize}
\end{theorem}

\begin{remark}
The works \cite{scaman2017optimal,scaman2018optimal,lu2021optimal} utilize the linear graph to establish the same relation $D = \Theta(1/\sqrt{1-\beta})$. 
However, given fixed network size $n$, the linear graph admits limited weight matrices with $\beta = \cos(\pi/n)$. 
\end{remark}

\begin{remark}
Sun and Hong \cite{sun2019distributed} propose a slightly wider class of networks named linear-star graph, which, however, leads to a suboptimal relation $D= \Omega(1/[n(1-\beta)])$ \cite[Proof of Corollary 3.1]{sun2019distributed}.
\end{remark}

\begin{remark}
It is worth noting that for deterministic decentralized optimization where $n$ may vary, one can tune $n$ to make $\beta = \cos(\pi/n)$ lie in the full interval $(0,1)$ as shown in \cite{scaman2017optimal,scaman2018optimal}. In other words, \cite{scaman2017optimal,scaman2018optimal}  proved that $D = \Theta(1/\sqrt{1-\beta})$ for any $\beta \in (0,1)$ if $n$ is a varying parameter. However, this result is not valid in stochastic scenarios where $n$ is required to be fixed.
\end{remark}

\vspace{ -1mm}
\section{Lower bounds with general weight matrices}
\vspace{ -1mm}
\label{sec-lower-bounds}
This section will establish the lower bounds for smooth and non-convex decentralized stochastic optimization. All proof details are in Appendix \ref{app-lower-bounds}. When the cost function does not satisfy the PL condition \eqref{PL-cond}, we will derive that the convergence rate is lower bounded by $\Omega(\frac{\sigma}{\sqrt{nT}} + \frac{D}{T})$. This result together with Theorem \ref{thm-D-beta-relation} leads to:
\begin{theorem}\label{thm-lower-bound-nc}
For any $L>0$, $n\geq 2$, $\beta \in[0,\cos(\pi/n)]$,  and $\sigma>0$, there exists a set of loss functions $\{f_i\}_{i=1}^n \subseteq \mathcal{F}_{L}$, a set of stochastic gradient oracles $\{\tilde{g}_i\}_{i=1}^n \subseteq \mathcal{O}_{\sigma^{2}}$, and a weight matrix $W \in \mathcal{W}_{n,\beta}$, such that it holds for any $A \in \mathcal{A}_{W}$ starting form $x^{(0)}$ that 
\begin{align}\label{eqn:lower-bound-nc}
  \mathbb{E}\|\nabla f(\hat{x}_{A, \{f_i\}_{i=1}^n,\{\tilde{g}_i\}_{i=1}^n,W,T})\|^2 =   \Omega\left(\frac{\sigma\sqrt{L}}{\sqrt{nT}}+ \frac{L}{T\sqrt{1-\beta}}\right).
\end{align}
\end{theorem}
 \begin{remark}[\bf Comparison with existing bounds]
Comparing to the lower bound established in \cite{lu2021optimal}, Theorem \ref{thm-lower-bound-nc} admits a much broader class of weight matrices with $\beta \in[0,\cos(\pi/n)]$. While the interval $[0, \cos(\pi/n)] \subset [0, 1)$, it approaches to $[0, 1)$ as $n$ goes large. 
 Such interval is broad enough to cover most weight matrices (generated through the Laplacian rule $W = I - L/d_{\max}$) resulted from common topologies such as line, ring, grid, torus, hypercube, exponential graph, complete graph, Erdos-Renyi graph, geometric random graph, etc. whose $\beta$ lies in the interval $[0, \cos(\pi/n)]$ when $n$ is sufficiently large, see details in Appendix \ref{app-lower-bounds}. 
 \end{remark}

\begin{remark}[\bf Linear speedup]
The first term $1/\sqrt{nT}$ dominates \eqref{eqn:lower-bound-nc} when $T$ is sufficiently large, which will require $T = O(1/(n\epsilon^2))$ (it is inversely proportional to $n$) iterations to reach accuracy $\epsilon$. Therefore, if the term involving $nT$ is dominating the rate for some $T$, we say the algorithm is in its linear-speedup stage. Linear speedup is a fundamental property that distributed algorithms enjoy.
\end{remark}
\begin{remark}[\bf Transient iteration complexity]\label{rm-transient-iteration}
Due to the topology-incurred overhead in convergence rate, a decentralized algorithm has to experience a transient stage to achieve linear-speedup. Transient iterations are referred to those iterations before an algorithm reaches  linear-speedup stage. To reach linear speedup in the lower bound \eqref{eqn:lower-bound-nc}, $T$ has to satisfy $\frac{1}{T\sqrt{1-\beta}} \le \frac{\sigma}{\sqrt{nT}}$ and hence $T = \cO(\frac{n}{(1-\beta)\sigma^2})$,

which is the transient iteration complexity shown in the first line in Table \ref{table-non-convex}. The smaller the transient iterations are, the faster the algorithm can achieve linear speedup. Transient iteration complexity is an important metric to evaluate
how sensitive the algorithm is to $\beta$. 
\end{remark}

When the cost function satisfies PL condition \eqref{PL-cond}, we prove the convergence rate is lower bounded by $\Omega(\frac{\sigma}{\sqrt{nT}} + \exp(-CT/D))$ where $C$ is some constant. This result together with Theorem \ref{thm-D-beta-relation} leads to:
\begin{theorem}\label{thm-lower-bound-pl}
For any $L>0$, $n\geq 2$, $\beta \in[0,\cos(\pi/n)]$,  and $\sigma>0$, there exists a set of loss functions $\{f_i\}_{i=1}^n \subseteq \mathcal{F}_{L,\mu}$, a set of stochastic gradient oracles $\{\tilde{g}_i\}_{i=1}^n \subseteq \mathcal{O}_{\sigma^{2}}$, and a weight matrix $W \in \mathcal{W}_{n,\beta}$, such that it holds for any $A \in \mathcal{A}_{W}$ starting from $x^{(0)}$ that 
\begin{small}
\begin{align}\label{eqn:lower-bound-PL}
\setlength{\abovedisplayskip}{3pt}\setlength{\belowdisplayskip}{3pt}
 \mathbb{E}[f(\hat{x}_{A, \{f_i\}_{i=1}^n,\{\tilde{g}_i\}_{i=1}^n,W,T}) - f^\star] = \Omega\left(\frac{\sigma^2}{\mu nT}+\frac{\mu}{L}\exp(-T\sqrt{\mu/L}\sqrt{1-\beta})\right). 
\end{align}
\end{small}
Rate \eqref{eqn:lower-bound-PL} results in an $\tilde{\cO}((1-\beta)^{-\frac{1}{2}})$ transient complexity.
\end{theorem}
Expression \eqref{eqn:lower-bound-PL} is the first lower bound for smooth and non-convex stochastic decentralized optimization under the PL condition to our knowledge. Please note that the above lower bound also applies to strongly-convex stochastic decentralized optimization, see the construction details in Appendix \ref{app-lower-bounds}.

\vspace{ -2mm}
\section{DSGD with multiple gossip communications}
\vspace{ -2mm}

This section will present a decentralized algorithm that achieve the lower bounds established in Sec.~\ref{sec-lower-bounds} up to constants $L$ and $\mu$. The new algorithm is a direct extension of the vanilla decentralized SGD (DSGD) \cite{chen2012diffusion,koloskova2020unified}. Inspired by the algorithm development in \cite{Rogozin2021AnAM,lu2021optimal}, we add two additional components to DSGD: gradient accumulation and multiple-gossip communication. The main recursions are listed in Algorithm \ref{Algorithm: MG-DSGD} which utilizes the fast gossip average step \cite{liu2011accelerated} in Algorithm \ref{Algorithm: fast-gossip}. We call the new algorithm as MG-DSGD where ``MG'' indicates ``multiple gossips''. All proofs are in Appendix \ref{app-mg-dsgd}. 

\begin{algorithm}[t]

	\caption{{D}ecentralized SGD with {m}ultiple {g}ossip steps (MG-DSGD)}
			\label{Algorithm: MG-DSGD}
			
			\textbf{Require:} Initialize $x_i^{(0)}=0$, the weight matrix $W$, and the rounds of gossip steps $R$.
			
			\vspace{1mm}
			\textbf{for} $k=0,1,2,..., K$, every node $i$ \textbf{do}
			
			\vspace{0.5mm}
			\hspace{5mm} 	{Sample $\{\xi^{(k,r)}_{i}\}_{r=1}^R$ independently and let  $g_i^{(k)} \hspace{-1mm}=\hspace{-1mm} \frac{1}{R}\hspace{-1mm}\sum\limits_{r=1}^R\hspace{-1mm}\nabla F(x_i^{(k)};\xi_i^{(k,r)})$;} $ \   \triangleright  \mbox{ \footnotesize{grad.  accumulation}}$ 
			
			\hspace{5mm} Update $\phi_i^{(k+1)} = x_i^{(k)} - \gamma g_i^{(k)}$; $\hspace{2 cm}  \triangleright  \mbox{ \footnotesize{local stochastic gradient descent}}$ \vspace{1mm}
			
			\hspace{5mm} {Update	$x^{(k+1)}_i = \textbf{FastGossipAverage}(\{\phi^{(k+1)}_i\}_{i=1}^n,{W}, R)$;} \quad 
			$\triangleright \mbox{ \footnotesize{multiple gossips}}$

\end{algorithm}
		
		\begin{algorithm}[t]
			\caption{$x_i =  \textbf{FastGossipAverage}(\{\phi_i\}_{i=1}^n,{W}, R)$}
			\label{Algorithm: fast-gossip}
			
			\textbf{Require:} $\{\phi_i\}_{i=1}^n$, ${W}$, $R$, and $\beta$; let  ${z_i^{(0)} = z_i^{(-1)}=\phi_i}$ and $\eta = \frac{1-\sqrt{1-\beta^2}}{1+\sqrt{1+\beta^2}}$.
			
			\vspace{1mm}
			\textbf{for} $r=0,1,2,..., R-1$, every node $i$ \textbf{do}
			
			\vspace{0.5mm}
			\hspace{5mm} Update	$z^{(r+1)}_i = (1+\eta)\sum_{j\in\mathcal{N}_i}{w}_{ij} z^{(r)}_i-\eta z_i^{(r-1)}$;
			$\hspace{0.83cm}   \triangleright \mbox{ \footnotesize{fast gossip averaging}}$
			
			\textbf{Output:}	$x_i = {z}^{(R)}_i$;
		\end{algorithm}

The \detag algorithm proposed in \cite{lu2021optimal} applies the gradient accumulation and multiple-gossip communication techniques to gradient tracking, which achieves the optimal rate in non-convex (without PL condition) stochastic decentralized optimization. This paper finds in \ours that the multiple gossip technique can significantly reduce the influence of data heterogeneity, and the gradient tracking step can be saved to achieve the convergence optimality. {Similar observations also appeared in \cite{Rogozin2021AnAM} albeit in the strongly-convex scenario}. Removing gradient tracking from algorithm updates brings two benefits. First, it saves half of the communication overhead and memory storage per iteration. Second, it enables numerous effective acceleration techniques designed exclusively for decentralized SGD such as Quasi-global momentum \cite{lin2021quasi}, DecentLaM \cite{yuan2021decentlam}, adaptive momentum \cite{nazari2019dadam}, etc. which cannot be easily extended to gradient tracking. 

Let $T$ be the toal number of gradient queries (or decentralized communications) at each node. Since each node takes $R$ gradient queries and $R$ communications at round $k$, it holds that $T = K R$ when \ours finishes after $K$ rounds. The following theorems clarify the convergence rate of \ours where $T = KR$. We omit the initialization constants below. Our analysis techniques relate to the classic SGD analysis but with inaccurate oracles in the nonconvex case \cite{dasdsafcvzx}.

\begin{theorem}\label{thm-MG-DSGD-rate-nc}
Given $L>0$, $n\geq 2$, $\beta \in[0,1)$, $\sigma>0$, and let $A$ denote Algorithm \ref{Algorithm: MG-DSGD}. Assuming that  $\frac{1}{n}\sum_{i=1}^n \|\nabla f_i(x) - \nabla f(x)\|^2 \le b^2$ for any $x\in \RR^{n}$ and the number of gossip rounds $R$ is set as in Appendix \ref{app-mg-dsgd}, the convergence of $A$ can be bounded for any $\{f_i\}_{i=1}^n \subseteq \cF_{L}$ and any $W \in \cW_{n,\beta}$ that
\begin{small}
\begin{equation}\label{MG-DSGD-rate-1}
\setlength{\abovedisplayskip}{3pt}\setlength{\belowdisplayskip}{3pt}
\frac{1}{K}\sum_{k=1}^K \mathbb{E}\|\nabla f(\bar{x}^{(k)})\|^2 = \tilde{O}\left( \frac{\sigma\sqrt{L}}{\sqrt{nT}} + \frac{L}{T\sqrt{1-\beta}} \right) \quad \mbox{where} \quad \bar{x}^{(k)} = \frac{1}{n}\sum_{i=1}^n x_i^{(k)}.
\end{equation}
\end{small}
\end{theorem}
\begin{remark}
The rate \eqref{MG-DSGD-rate-1} matches with the lower bound \eqref{eqn:lower-bound-nc} up to logarithm factors. Furthermore, we remark that the gradient similarity assumption $\frac{1}{n}\sum_{i=1}^n \|\nabla f_i(x) - \nabla f(x)\|^2 \le b^2$ is not required to obtain the lower bound in Theorem \ref{thm-lower-bound-nc}. Due to these two restrictions, \ours is a nearly-optimal (not fully-optimal) algorithm. However, while the analysis in \ours needs to assume bounded gradient  similarity $b^2$, it only affects the convergence rate as a logarithm term hidden in the $\tilde{\cO}(\cdot)$ notation. If one wants to completely removes the gradient similarity assumption from Theorem \ref{thm-MG-DSGD-rate-nc}, he/she can refer to the techniques in gradient tracking \cite{lu2021optimal,xin2020improved,lu2019gnsd}.
\end{remark}
The rate \eqref{MG-DSGD-rate-1} matches with the lower bound \eqref{eqn:lower-bound-nc} up to logarithm factors. The comparison between \ours with other state-of-the-art algorithms is listed in Table \ref{table-non-convex}. \ours matches DeTAG in convergence rate and saves half of the communications, see experiments in Fig.~\ref{fig:cifar10_alpha1}. While the analysis in \ours needs to assume bounded data heterogeneity $b^2$, it only affects the convergence rate as a logarithm term hidden in the $\tilde{\cO}(\cdot)$ notation.

\begin{theorem}\label{thm-MG-DSGD-rate-pl}
	Under the same assumptions as in Theorem \ref{thm-MG-DSGD-rate-nc} and setting $R$ as in Appendix \ref{app-mg-dsgd}, the convergence of $A$ can be bounded for any loss functions $\{f_i\}_{i=1}^n \subseteq \cF_{L,\mu}$, and any $W \in \cW_{n,\beta}$ by
	\begin{small}
	\begin{align}\label{MG-DSGD-rate-2}
	\setlength{\abovedisplayskip}{3pt}\setlength{\belowdisplayskip}{3pt}
 \frac{1}{K}\sum_{k=1}^K \mathbb{E}[f(\bar{x}^{(k)}) - f^\star] = {\tilde{O}\left(\frac{L\sigma^2}{{\mu^2 nT}} + \exp(-\mu T\sqrt{1-\beta}/L) \right)}.
	\end{align}
	\end{small}
\end{theorem}
\begin{remark}
Comparing \eqref{MG-DSGD-rate-2} and lower bound \eqref{eqn:lower-bound-PL}, we find its dependence on $\sigma^2$, $n$, $T$, and $1-\beta$ is optimal while that on $L$ and $\mu$ is not. One reason is that our algorithm does not involve any Nesterov acceleration \cite{nesterov2003introductory} mechanism. 
The established upper bound in a recent work \cite{Rogozin2021AnAM} is closer to the lower bound \eqref{eqn:lower-bound-PL} by utilizing Nesterov acceleration, see Table \ref{table-non-convex-PL-app} in Appendix \ref{app-detailed-comp}.
However, it requires an increasingly large mini-batch of data per iteration, and only works for the strongly-convex scenario. It is not known whether the acceleration mechanisms in \cite{Rogozin2021AnAM} can be extended to the non-convex and PL scenario. The bound \eqref{MG-DSGD-rate-2} is state-of-the-art under the general PL condition to our knowledge. 
\end{remark}

\vspace{ -2mm}
\section{Experiments}
\label{sec-experiments}
\vspace{ -2mm}

This section will validate our theoretical results by empirically comparing different decentralized algorithms \dsgd\cite{lian2017can}, \dsquare\cite{tang2019doublesqueeze}, \dsgt\cite{zhang2019decentralized}, \detag\cite{lu2021optimal} and \ours in deep learning. 

\textbf{Implementation setup.} We implement all decentralized algorithms with PyTorch \cite{paszke2019pytorch} 1.6.0 using NCCL 2.8.3 (CUDA 10.1) as the communication backend, in which the partial averaging primitive is implemented by NCCL send/receive primitives. For parallel SGD, we used PyTorch’s native Distributed Data Parallel (DDP) module. All the models and training scripts in this section run on servers with 8 NVIDIA V100 GPUs with each GPU treated as one node. Similar to existing works \cite{lu2021optimal}, we use Ring graph throughout all experiments.

\textbf{Tasks and training details.} A series of experiments are carried out with \cifar\cite{krizhevsky2009learning} and \imagenet\cite{deng2009imagenet} to compare the aforementioned methods. For \cifar dataset, it consists of 50,000 training images and 10,000 validation images in 10 classes. For \imagenet dataset, it consists of 1,281,167 training images and 50,000 validation images in 1000 classes. We utilize ResNet variants \cite{he2016deep} on \cifar (ResNet-20 with 0.27M parameters and  ResNet-18 with 11.17M parameters) and ResNet-50 on \imagenet. For training protocol, we train total 300 epochs and the learning rate is warmed up in the first 5 epochs and is decayed by a factor of 10 at 150 and 250-th epoch. For learning rate, we tuned a strong baseline in the \psgd setting (5e-3 for single node) and used the same setting in all decentralized methods. The batch size is set to 128 on each node. All \cifar experiments are repeated three times with different seeds. For fair comparison, we maintain the total number of gradient queries the same for algorithms in experiments, i.e., $T = KR$. 

\begin{table}[t]
\tablefontsize
\caption{Accuracy comparison on \cifar with homogeneous data (8 nodes).}
\vskip -1mm
\begin{center}
\begin{small}

\begin{sc}

\setlength{\tabcolsep}{3.25mm}{

\begin{tabular}{ccccc}
\toprule
 &  \multicolumn{2}{c}{With Momentum} &  \multicolumn{2}{c}{Without Momentum} \\

Methods   & ResNet18     & ResNet20   & ResNet18     & ResNet20   \\

\midrule

\psgd & 93.99 ± 0.52 & 91.62 ± 0.13  & 93.22 ± 0.19 & 89.06 ± 0.08  \\ 
\dsgd & 94.54 ± 0.05 & 91.46 ± 0.06  & 93.67 ± 0.04 & 89.14 ± 0.08   \\ 
\dsgt & 94.34 ± 0.04 & 91.64 ± 0.01 &  93.57 ± 0.02 & 88.85 ± 0.11 \\
\dsquare & 92.51 ± 0.07 & 89.83 ± 0.08  & 92.18 ± 0.08 & 88.67 ± 0.11 \\ 
\detag & 94.40 ± 0.13 & \textbf{91.77 ± 0.25}  & 93.17 ± 0.18 & 88.97 ± 0.11 \\ 
\ours & \textbf{94.57 ± 0.05} & \textbf{91.77 ± 0.09} & \textbf{93.75 ± 0.12}  & \textbf{89.18 ± 0.08} \\ 

\bottomrule
\end{tabular}}
\vskip -4mm
\end{sc}
\end{small}
\end{center}
\label{table-cifar10-performance}
\vskip -1mm
\end{table}

\begin{table}[t!]
    \centering
    \vskip -2mm    
    \caption{Accuracy comparison on \cifar with heterogeneous data (ResNet-20).} 
    \vskip -1mm    
    \label{tab:cifar10_heter_data}
    \begin{small}
\begin{sc}
\setlength{\tabcolsep}{1.25mm}{

\begin{tabular}{cccccc}
\toprule
Methods  &  \dsgd & \dsgt & \dsquare & \detag & \ours         \\
\midrule
$\alpha=0.1$ & 58.14 ± 1.76 & 55.16 ± 0.42 & 53.85 ± 0.40
 & 58.77 ± 1.30 & \textbf{59.42 ± 3.36} \\
$\alpha=1$  & 84.74 ± 0.97 & 84.74 ± 0.32 & 84.1 ± 0.64 & 84.82 ± 0.19 &  \textbf{84.91 ± 0.74} \\
$\alpha=10$  & 90.19 ± 0.47 & 89.83 ± 0.12 & 89.80 ± 0.25 & 90.27 ± 0.08 &  \textbf{90.43 ± 0.04} \\

\bottomrule
\end{tabular}}%
\end{sc}
\end{small} 
\end{table}

\begin{figure}[!h]
    \vskip -2mm
    \centering
    \includegraphics[width=0.35\textwidth]{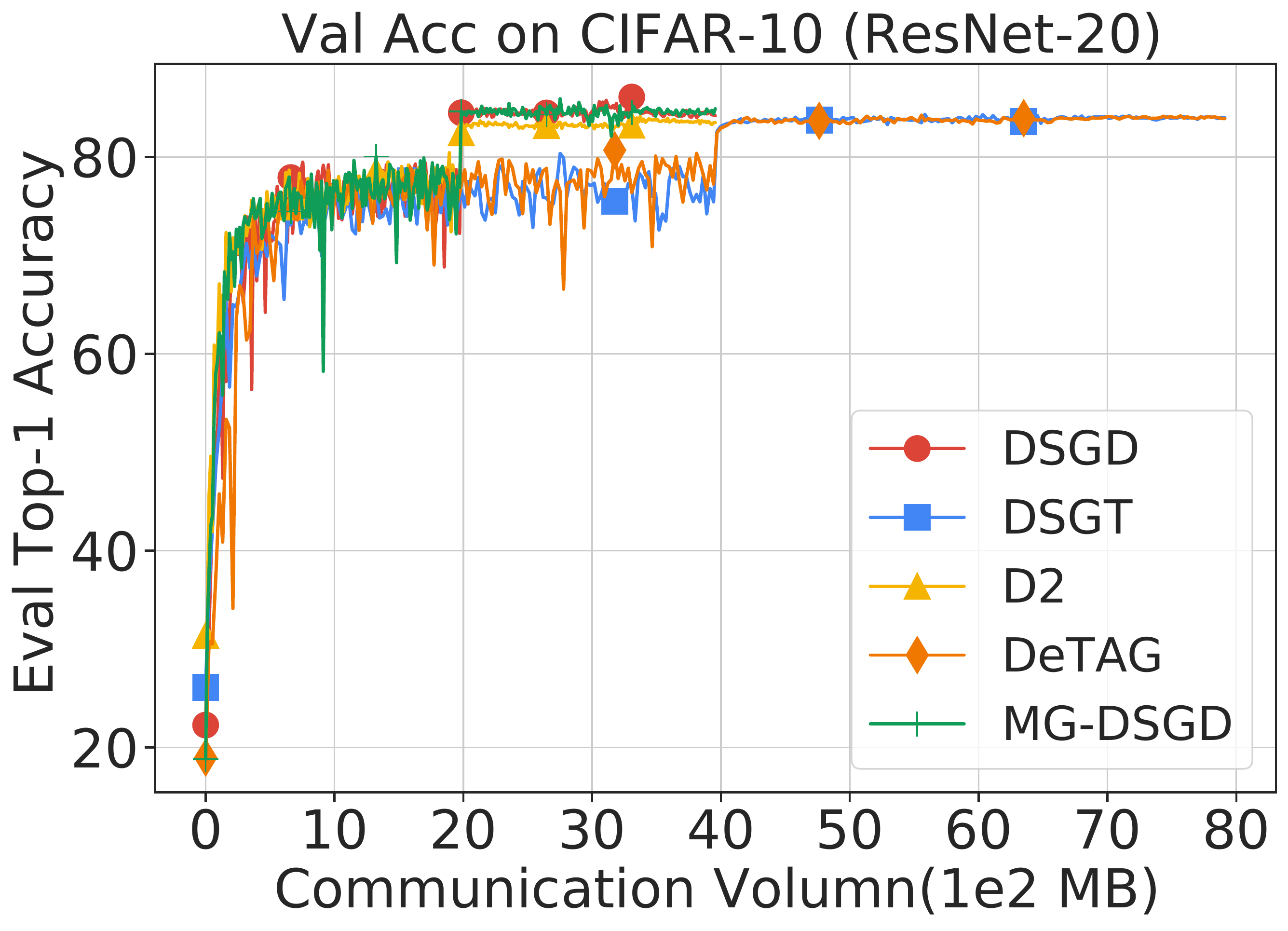}
    \includegraphics[width=0.35\textwidth]{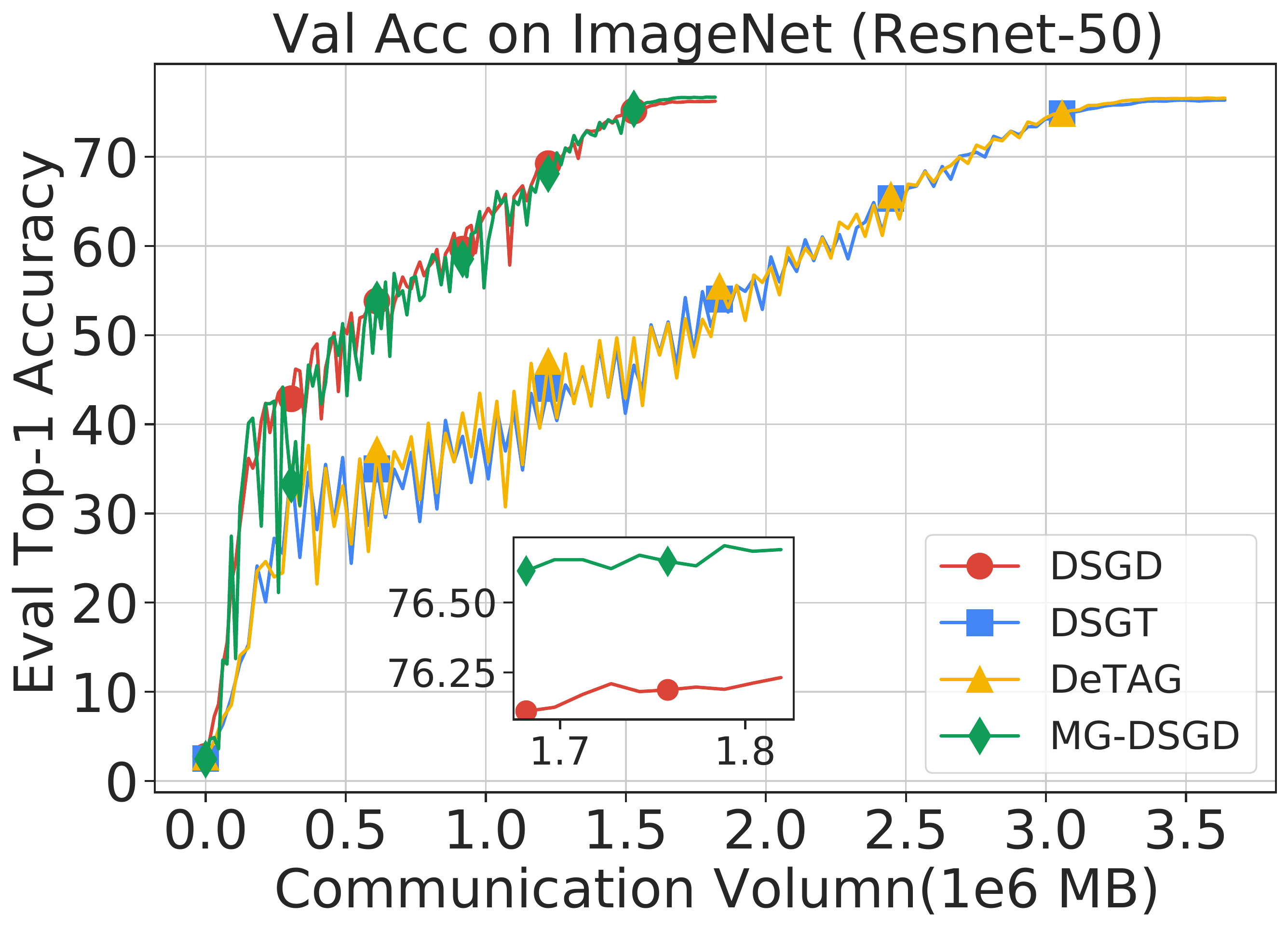}
    \vskip -3mm
    \caption{Accuracy in terms of the communication volumes. Left: CIFAR-10 with heterogeneous data ($\alpha=1$); Right: ImageNet with homogeneous data.}
    \label{fig:cifar10_alpha1}
    \vskip -5mm
\end{figure}

\textbf{Performance with homogeneous data.} Table~\ref{table-cifar10-performance} compares the proposed \ours with centralized parallel SGD and other baselines with homogeneous data distributions. For \detag and \ours, the communication round $R$ is set as 2. We also performed an  ablation study on the effect of momentum acceleration. Except for the ResNet-20 scenario with  momentum in which \ours and DeTAG achieve the same performance, \ours consistently outperforms baseline methods in other scenarios. Moreover, \ours is more communication-efficient than DSGT and DeTAG.  

\noindent\textbf{Performance with heteogeneous data.} Following previous works \cite{yurochkin2019bayesian,lin2021quasi}, we simulate data heterogeneity among nodes with a Dirichlet distribution-based partitioning. A quantity $\alpha$ controls the degree of the data heterogeneity. The training data for a particular class tends to concentrate in a single node as $\alpha\rightarrow 0$, i.e. becoming more heterogeneous, while the homogeneous data distribution is achieved as $\alpha\rightarrow \infty$. We tested $\alpha=0.1/1/10 $ for all compared methods as corresponding to a setting with high/mild/low heterogeneity. As shown in Table~\ref{tab:cifar10_heter_data}, DeTAG and \ours outperform the other baselines by visible margins, and \ours is with slightly better accuracies. The left plot in Fig.~\ref{fig:cifar10_alpha1} illustrates that \ours achieve a competitive accuracy to DeTAG using much less communication budgets on CIFAR-10 with $\alpha = 1$.

\begin{wraptable}{r}{4cm}
\vskip -5mm
\caption{Performance of different methods (ImageNet).} 

\label{tab:imageNet_acc}
\centering
\begin{small}

\begin{tabular}{cc}
\toprule
Methods & Acc. \\

\midrule
\dsgd & 76.23 ± 0.06 \\
\dsgt & 76.37 ± 0.07 \\
\detag & 76.57 ± 0.02 \\
\ours & \bf{76.64 ± 0.03}   \\

\bottomrule
\end{tabular}

\end{small} 
\vskip -5mm
\end{wraptable}

\noindent\textbf{Large-batch training in ImageNet.} To train ResNet-50 on ImageNet, utilization of the large-batch is necessary to speedup the training process. Since \ours is based on DSGD, we can extend the large-batch acceleration techniques introduced in \cite{yuan2021decentlam,lin2021quasi} to \ours to achieve better performance. These techniques are developed exclusively for DSGD and cannot be integrated to DeTAG easily. Following \cite{yuan2021decentlam}, we train total 120 epochs. The learning rate is warmed up in the first 20 epochs and is decayed in a cosine annealing scheduler. The results are listed in Table~\ref{tab:imageNet_acc}, showing that \ours has better accuracy than other approaches. Fig.~\ref{fig:cifar10_alpha1}(right) illustrates how different algorithms perform in terms of communication budgets used.

\noindent \textbf{More results.} In Appendix \ref{app-exp}, we depict the accuracy curves in terms of epochs (rather than communication volumes). The ablation on the influence of inner loop $R$ is also listed there. 
\vspace{ -3mm}
\section{Conclusion}
\vspace{ -3mm}

This paper revisits the theoretical limits in smooth and non-convex stochastic decentralized optimization. Existing lower bound proved in \cite{lu2021optimal} is only valid for special weight matrices with connectivity measure $\beta = \cos(\pi/n)$. It is not known what the lower bound is when using  weight matrices with other $\beta$. To address this issue, this paper establishes the lower bound for any weight matrix with $\beta \in [0, \cos(\pi/n)]$. We also clarify the lower bound when loss functions satisfy the PL condition. A new algorithm is proposed to match these lower bounds up to logarithm factors. 

\section*{Acknowledgements}
The authors are grateful to all anonymous reviewers in NeurIPS 2022 for the very helpful discussions during the rebuttal process as well as useful references on accelerated decentralized gradient descent.

{
\bibliographystyle{ieee_fullname}
\bibliography{references}
}

\newpage
\appendix 

\section{Preliminary}
\label{app-preliminary}
\vspace{2mm}
\noindent \textbf{Notation.} We first introduce necessary notations as follows.
\begin{itemize}
	\item $\vx^{(k)} = [(x_1^{(k)})^T; (x_2^{(k)})^T; \cdots; (x_n^{(k)})^T]\in \mathbb{R}^{n\times d}$
	\item $\nabla F(\vx^{(k)};\xi^{(k)}) = [\nabla F_1(x_1^{(k)};\xi_1^{(k)})^T; \cdots; \nabla F_n(x_n^{(k)};\xi_n^{(k)})^T]\in \mathbb{R}^{n\times d}$
	\item $\nabla F(\vx^{(k)};\xi^{(k,r)}) = [\nabla F_1(x_1^{(k)};\xi_1^{(k,r)})^T; \cdots; \nabla F_n(x_n^{(k)};\xi_n^{(k,r)})^T]\in \mathbb{R}^{n\times d}$
	\item $\nabla f(\vx^{(k)}) = [\nabla f_1(x_1^{(k)})^T; \nabla f_2(x_2^{(k)})^T; \cdots; \nabla f_n(x_n^{(k)})^T]\in \mathbb{R}^{n\times d}$ 
	\item $\bar{\vx}^{(k)} = [(\bar{x}^{(k)})^T; (\bar{x}^{(k)})^T; \cdots; (\bar{x}^{(k)})^T]\in \mathbb{R}^{n\times d}$ where $\bar{x}^{(k)} = \frac{1}{n}\sum_{i=1}^n x_i^{(k)}$
	\item $W=[w_{ij}]\in \mathbb{R}^{n\times n}$ is the weight matrix.
	\item $\mathds{1}_n = \mathrm{col}\{1,1,\cdots, 1\} \in \RR^n$.
	\item Given two matrices $\vx, \vy \in \RR^{n\times d}$, we define inner product $\langle \vx, \vy \rangle = \mathrm{tr}(\vx^T \vy)$ and the Frobenius norm $\|\vx\|_F^2 = \langle \vx, \vx \rangle$. 
	\item Given $W\in \RR^{n\times n}$, we let $\|W\|_2 = \sigma_{\max}(W)$ where   $\sigma_{\max}(\cdot)$ denote the maximum sigular value.
\end{itemize}

\vspace{2mm}
\noindent \textbf{DSGD in matrix notation.} The recursion of DSGD can be written in matrix notation: 
\begin{align}
\vx^{(k+1)}  = 
W \big(\vx^{(k)} - \gamma\nabla F(\vx^{(k)}; \xi^{(k)}) \big)
\label{eq:DSGD_global_average}
\end{align}

\vspace{2mm}
\noindent \textbf{\ours in matrix notation.} The main recursion of \ours can be written in matrix notation: 
\begin{align}
\vg^{(k)} &= \frac{1}{R}\sum_{r=1}^R  \nabla F(\vx^{(k)}; \xi^{(k,r)}) \\
\vx^{(k+1)} & = 
\bar{M} (\vx^{(k)} - \gamma \vg^{(k)} )
\label{eq:MG_DSGD_global_average}
\end{align}
where the weight matrix $\bar{M} = M^{(R)}$ and $M^{(R)}$ is achieved via the fast gossip averaging loop:
\begin{align}
	M^{(-1)} &= M^{(0)} = I \label{eqn:M1}\\
	M^{(r+1)} &= (1+\eta) W M^{(r)} - \eta M^{(r-1)}, \quad \forall\ r=0,1,\cdots, R-1.\label{eqn:M2}
\end{align}

\vspace{2mm}
\noindent \textbf{Smoothness.} Since each $f_i(x) \in \cF_{L}$ is $L$-smooth, it holds that $f(x) = \frac{1}{n}\sum_{i=1}^n f_i(x)$ is also $L$-smooth. As a result, the following inequality holds for any $\x, \y \in \mathbb{R}^d$:
\begin{align}
f_i(\x) - f_i(\y) - \frac{L}{2}\|\x - \y\|^2 &\le  \langle \nabla f_i(\y), \x- \y \rangle \label{sdu-2}
\end{align}

\noindent \textbf{Network weighting matrix.} Since the weight matrix $W\in \cW_{n,\beta}$, it holds that 
\begin{align}\label{network-inequaliy}
\|W - \frac{1}{n}\mathds{1}_n\mathds{1}_n^T\|_2 = \beta.
\end{align}
Furthermore, one can establish the following upper result for the multi-gossip matrix $\bar{M}$.
\begin{proposition}[Proposition 3 of \cite{liu2011accelerated}]\label{prop:fast-gossip}
Let $M^{(r)}$ defined by iterations \eqref{eqn:M1} and \eqref{eqn:M2}, then it holds that for any $r=0,\dots,R$,
\begin{align}
    M^{(r)}\one_n=[M^{(r)}]^T\one_n=\one_n,\quad \mbox{and}\quad \|M^{(r)}-\frac{1}{n}\one_n\one_n^T\|_2\leq \sqrt{2}\Big(1-\sqrt{1-\beta}\Big)^r.
\end{align}
\end{proposition}
Therefore, when $r$ grows,  $M^{(r)}$ exponentially converges to $\frac{1}{n}\one_n\one_n^T$.

\noindent \textbf{Submultiplicativity of the Frobenius norm.} Given matrices $W\in \RR^{n\times n}$ and $\vy\in \RR^{n\times d}$, it holds that 
\begin{align}\label{submulti}
\|W\vy\|_F \le \|W\|_2 \|\vy\|_F.
\end{align}
To verify it, by letting $y_j$ be the $j$-th column of $\vy$, we have $\|W\vy\|_F^2 = \sum_{j=1}^d \|Wy_j\|_2^2 \le \sum_{j=1}^d \|W\|_2^2 \|y_j\|_2^2=\|W\|_2^2\|\vy\|_F^2$. 

\section{More detailed comparison}
\label{app-detailed-comp}
\begin{table}[h!]
\centering 
\caption{\small Rate comparison in smooth and non-convex stochastic decentralized optimization. Parameter $n$ denotes the number of all computing nodes, $\beta\in [0,1)$ denotes the connectivity measure of the weight matrix, $\sigma^2$ measures the gradient noise, $b^2$ denotes data heterogeneity, {$L$ denotes smoothness constant}, and $T$ is the number of iterations. The definition of transient iteration complexity can be found in Remark \ref{rm-transient-iteration} (the smaller the better). ``LB'' is lower bound while ``UB'' is upper bound. Notation $\tilde{O}(\cdot)$ hides all logarithm factors.}
\begin{tabular}{rllll}
\toprule
                    & \textbf{References}                                                                           & \textbf{Gossip matrix}                     & \textbf{Convergence rate}    & \hspace{-2mm} \textbf{Tran. iters.}\\ \midrule
\multirow{2}{*}{LB} & \cite{lu2021optimal}                                                & $\beta = \cos(\pi/n)$ &         $\Omega\big( \frac{\sigma{\sqrt{L}}}{\sqrt{nT}}+ \frac{{L}}{T (1-\beta)^{\frac{1}{2}}} \big)$              &            $O(\frac{nL}{(1-\beta)\sigma^2})$                            \vspace{1mm}   \\
 & {\color{blue}Theorem \ref{thm-lower-bound-nc}}                                                                           & {\color{blue}$\boldsymbol{\beta \in [0, \cos(\pi/n)]}$}       &          {\color{blue}$\Omega\big( \frac{\sigma{\color{blue}\sqrt{L}}}{\sqrt{nT}}+ \frac{{\color{blue}L}}{T (1-\beta)^{\frac{1}{2}}} \big)$}            &            {\color{blue}$O(\frac{nL}{(1-\beta)\sigma^2})$}                                    \\ \midrule
\multirow{5}{*}{UB} & DSGD \cite{koloskova2020unified}                                    & $\beta \in [0, 1)$           &  \hspace{-1.2cm}    $O\big( \frac{\sigma \sqrt{L}}{\sqrt{nT}} \hspace{-1mm}+\hspace{-1mm} \frac{\sigma^{\frac{2}{3}}{L^{\frac{2}{3}}}}{T^{\frac{2}{3}}(1-\beta)^{\frac{1}{3}}} \hspace{-1mm}+\hspace{-1mm} \frac{b^{\frac{2}{3}}{L^{\frac{2}{3}}}}{T^{\frac{2}{3}}(1-\beta)^{\frac{2}{3}}} \big)$                &                    ${O}\big(\frac{n^3{L}}{(1-\beta)^2\sigma^2}\big)^\dagger$          \vspace{1mm}    \\
                             & D$^2$/ED \cite{tang2018d}                   & $\beta \in [0, 1)$           &        $O\big( \frac{\sigma{\sqrt{L}}}{\sqrt{nT}} + \frac{{L}}{T (1-\beta)^3} \big)$                &       ${O}(\frac{n{L}}{(1-\beta)^6\sigma^2})$                                          \vspace{1mm}\\
                             & DSGT \cite{koloskova2021improved} & $\beta \in [0, 1)$           & $\tilde{O}\big( \frac{\sigma{\sqrt{L}}}{\sqrt{nT}} + \frac{\sigma^{\frac{2}{3}}{L^{\frac{2}{3}}}}{T^{\frac{2}{3}}(1-\beta)^{\frac{1}{3}}}  \big)$  &        $\tilde{O}(\frac{n^3{L}}{(1-\beta)^2\sigma^2})$                                     \vspace{1mm}     \\
                             & DeTAG \cite{lu2021optimal}                                          & $\beta \in [0, 1)$           & $\tilde{O}\big( \frac{\sigma{\sqrt{L}}}{\sqrt{nT}}+ \frac{{L}}{T (1-\beta)^{\frac{1}{2}}} \big)$ &                  $\tilde{O}(\frac{n{L}}{(1-\beta)\sigma^2})$                           \vspace{1mm}   \\
                             & {\color{blue}MG-DSGD}                                                                              & {\color{blue}$\beta \in [0, 1)$}           & {\color{blue}$\tilde{O}\big( \frac{\sigma{\color{blue}\sqrt{L}}}{\sqrt{nT}}+ \frac{L}{T (1-\beta)^{\frac{1}{2}}}\big)$} &                       {\color{blue}$\tilde{O}(\frac{nL}{(1-\beta)\sigma^2})$}                          \\ \bottomrule
\multicolumn{5}{l}{$^\dagger$\,\footnotesize{The complete complexity is ${O}(\frac{n^3{L}}{(1-\beta)^2\min\{\sigma^2,(1-\beta)^2 b^2\}})$, which reduces to ${O}(\frac{n^3{L}}{(1-\beta)^2\sigma^2})$ for small $\sigma^2$.}}
\end{tabular}
\label{table-non-convex-app}
\end{table}

\begin{table}[h!]
\centering 
\caption{\small Rate comparison between different algorithms in smooth and non-convex stochastic decentralized optimization under the PL condition. {Quantity $\mu$ is the PL constant.}}
\begin{tabular}{rllll}
\toprule
& \textbf{References}                                                                           & \textbf{Gossip matrix}                     & \textbf{Convergence rate}    & \hspace{-2mm} \textbf{Tran. iters.}\\ \midrule
LB
 & {\color{blue}Theorem \ref{thm-lower-bound-pl}}                                                                           & {\color{blue}$\beta \in [0, \cos(\pi/n)]$}       &          {\color{blue}$\Omega\big( \frac{\sigma^2}{\mu nT}+ \frac{\mu}{L} e^{-T\sqrt{\mu(1-\beta)/L}} \big)$}            &            {\color{blue}$\tilde{O}(\frac{(L/\mu)^{1/2}}{(1-\beta)^{1/2}})$}                                    \\ \midrule
\multirow{5}{*}{UB} & DSGD \cite{koloskova2020unified}$^\dagger$                                    & $\beta \in [0, 1)$           &  \hspace{-1.2cm}    $\tilde{O}\big( \frac{\sigma^2}{{\mu}nT} \hspace{-0.5mm}+\hspace{-0.5mm} \frac{{L}\sigma^2}{{\mu^2}T^2(1-\beta)} \hspace{-0.5mm}+\hspace{-0.5mm} \frac{{L}b^2}{{\mu^2} T^2 (1-\beta)^{2}} \big)$                &                    $\hspace{-8mm}\tilde{O}(\frac{n{L/\mu}}{(1-\beta)\min\{1,(1-\beta) b^2\}})$           \vspace{1mm}    \\
& {DAGD\cite{Rogozin2021AnAM}$^\dagger$$^\ddag$}                 & {$\beta \in [0, 1)$}       &        \hspace{-1.2cm}  {$\tilde{O}\big( \frac{\sigma^2}{\mu nT}+  e^{-T\sqrt{\mu(1-\beta)/L}} \big)$}            &       \hspace{-3mm}     {$\tilde{O}(\frac{(L/\mu)^{1/2}}{(1-\beta)^{1/2}})$}                                 \vspace{1mm}    \\
 & D$^2$/ED \cite{yuan2021removing}$^\dagger$                   & $\beta \in [0, 1)$           &      \hspace{-1.2cm}    $\tilde{O}\big( \frac{\sigma^2}{{\mu}nT} \hspace{-0.5mm}+\hspace{-0.5mm} \frac{{L}\sigma^2}{{\mu^2}T^2(1-\beta)} \hspace{-0.5mm}+\hspace{-0.5mm} e^{-{\mu}T(1-\beta)/{L}} \big)$               &      \hspace{-3mm}   $\tilde{O}(\frac{n{L/\mu}}{1-\beta})$                                          \vspace{1mm}\\
 & DSGT \cite{alghunaim2021unified} & $\beta \in [0, 1)$           &  \hspace{-1.2cm}    $\tilde{O}\big( \frac{{L}\sigma^2}{{\mu^2}nT} \hspace{-0.5mm}+\hspace{-0.5mm} \frac{{L^2}\sigma^2}{{\mu^3}T^2(1-\beta)} \hspace{-0.5mm}+\hspace{-0.5mm} e^{-{\mu}T(1-\beta)/{L}} \big)$       &        \hspace{-3mm}  $\tilde{O}(\frac{n{L/\mu}}{1-\beta})$                                              \vspace{1mm}     \\
 & DSGT \cite{xin2020improved}                                          & $\beta \in [0, 1)$    &       \hspace{-1.2cm}    $\tilde{O}\big( \frac{{L}\sigma^2}{{\mu^2}nT} \hspace{-0.5mm}+\hspace{-0.5mm} \frac{{L^2}\sigma^2}{{\mu^3}T^2(1-\beta)^3} \hspace{-0.5mm}+\hspace{-0.5mm} e^{-{\mu}T(1-\beta)/{L}} \big)$  &        \hspace{-3mm}           $\tilde{O}(\frac{n{L/\mu}}{(1-\beta)^3})$                           \vspace{1mm}   \\
 & {\color{blue}MG-DSGD}                                                                              & {\color{blue}$\beta \in [0, 1)$}           & {\color{blue}$\tilde{O}\big( \frac{L\sigma^2}{\mu^2 nT}+ e^{-T\mu(1-\beta)^{\frac{1}{2}}/L} \big)$}   &               \hspace{-3mm}        {\color{blue}$\tilde{O}(\frac{L/\mu}{(1-\beta)^{1/2}})$}                       \\ \bottomrule  
\multicolumn{5}{l}{$^\dagger$\,\footnotesize{{These rates are} derived under the strongly-convex assumption, {not the general PL condition.}}}\\
\multicolumn{5}{l}{$^\ddag$\,\footnotesize{{This rate is achieved by utilizing increasing (non-constant) mini-batch sizes.} }}\\
\end{tabular}
\label{table-non-convex-PL-app}
\end{table}

\section{Ring-Lattice graph}
\label{app-ring-lattice}

For any $2\leq k<n-1$ and $k$ is even, the  ring-lattice graph $\cR_{n,k}$ has several preferable properties. The following lemma  estimates the order of the diameter of $\cR_{n,k}$. 

\textbf{Lemma \ref{lm-diameter} (Formal version)} \textit{
For any two nodes $1\leq i< j\leq n$, the distance between $i$ and $j$ is  $ \lceil \frac{2\min\{j-i, i+n-j\} }{k}\rceil$. In particularly, the diameter of $\cR_{n,k}$ is $D_{n,k}=\lceil \frac{2\lfloor n/2\rfloor}{k}\rceil=\Theta(\frac{n}{k})$.}
\begin{proof}
By symmetry of $\cR_{n,k}$, it suffices to consider the case that $j-i\leq i+n-j$. On the one side,  the path $\{i, k/2+i,k+i,\dots,   (\lceil{2(j-i)}/{k}\rceil-1) k/2 +i,j\}$ {connects} the pair $(i,j)$ with length $\lceil{2(j-i)}/{k}\rceil$, which leads to $\dist(i,j)\leq \lceil{2 (j-i)}/{k}\rfloor $. On the other side, let $\{i_0\triangleq i,i_1,\dots,i_{m-1},i_m\triangleq j \}$ be one of the shortest paths connecting $(i,j)$. Without loss of generality, we assume $i<i_1<\dots<i_{m-1}<j$. Since $|i_{r}-i_{r-1}|\leq k/2$ for any $r=1,\dots, m$ by the definition of $\cR_{n,k}$, it holds that 
\[
\mathrm{dist}(i,j)=m\geq \frac{|i_m-i_0|}{k/2}=\frac{2(j-i)}{k}.
\]
Since $\mathrm{dist}(i,j)$ is a integer, we reach $\mathrm{dist}(i,j)\geq \lceil \frac{2(j-i)}{k}\rceil$. The diameter is readily obtained by maximizing the expression of $\dist(i,j)$ with respect $i$ and $j$.
\end{proof}

The following lemma clarifies the Laplacian matrix associated with $\cR_{n,k}$ and its eigenvalues.

\begin{lemma}\label{prop-laplacian}
The Laplacian matrix of  $\cR_{n,k}$  is given by $L_{n,k}=kI-A_{n,k}$ with adjacency matrix $A_{n,k}=\sum_{\ell=1}^{k/2}(J^\ell+(J^T)^\ell)$ where
\[
J=\left[\begin{array}{ccccc}
0 & 1 &0 & \cdots&0 \\
0  & 0 & 1 &\cdots &0 \\
\vdots & \ddots & \ddots & \ddots & \vdots\\
0&0 & \ddots & 0 & 1 \\
1&0 &\cdots & 0 & 0
\end{array}\right]\in\RR^{n\times n}\quad \text{and}   \quad J^T=J^{-1}\in\RR^{n\times n} .
\]
Moreover, the eigenvalues of the Laplacian matrix $L_{n,k}$ are given by
\begin{align}\label{znbznzn1}
\mu_j\triangleq k+1- \sum_{\ell=-k/2}^{k/2}\cos(\frac{2\pi (j-1)\ell }{n})\in[0,2k], \quad \forall~j\in\{1\dots,n\}
\end{align}
where $\mu_j$ is the $j$-th eigenvalue of $L_{n,k}$. In addition, it holds that 
\begin{equation}
(1-\pi^2/12)\frac{\pi^2k(k+1)(k+2)}{6n^2}\leq\min_{2\leq j\leq n}\mu_j\leq \frac{\pi^2k(k+1)(k+2)}{6n^2}.\label{eqn:gjfoeqf}
\end{equation}
\end{lemma}
\begin{proof}
	(\emph{Laplacian matrix.})  Since every node in $\cR_{n,k}$ is of degree $k$, the degree matrix of $\cR_{n,k}$ is $k I$. The adjacency matrix $A_{n,k}$ can be easily achieved by following the construction of $\cR_{n,k}$.
	
	(\emph{Eigenvalues.})  Let $\omega:=e^{2\pi i/n}$ (where $i$ is the imaginary number), then $J$ can be decomposed as $J=U \diag(\omega^{0}=1,\omega,\dots, \omega^{n-1})U^H$ with a unitary matrix $U=\frac{1}{\sqrt{n}}[\omega^{(p-1)(q-1)}]_{p,q=1}^n$. Since $J^\ell=U \diag(1,\omega^\ell,\dots, \omega^{\ell(n-1)})U^H$ for all $-k/2\leq \ell\leq k/2$, and $J^T = J^H=J^{-1}=U \diag(1,\bar{\omega}\dots, \bar{\omega}^{n-1})U^H$ where $\bar{\omega}=e^{-2\pi i/n}=\omega^{-1}$, we have that 
	\begin{align*}
	I+A_{n,k}&=\sum_{\ell=-k/2}^{k/2}J^{\ell}=U\left(\sum_{\ell=-k/2}^{k/2}\diag(1,\omega^\ell,\dots, \omega^{\ell(n-1)})\right)U^H\\
	&=U\left(\sum_{\ell=-k/2}^{k/2}\diag(1,\cos(\frac{2\pi \ell}{n}),\dots, \cos(\frac{2\pi \ell(n-1) }{n})\right)U^H\\
	&=U\diag\left(k+1, \sum_{\ell=-k/2}^{k/2} \cos(\frac{2\pi \ell}{n}), \dots,  \sum_{\ell=-k/2}^{k/2} \cos(\frac{2\pi \ell(n-1)}{n})\right)U^H.
	\end{align*}
	Therefore, the eigenvalues of $L_{n,k}$ are real numbers $\{\mu_j\triangleq k+1-\sum_{\ell=-k/2}^{k/2}\cos(\frac{2\pi (j-1)\ell }{n}):j=1,\dots,n\}$. It is easy to verify that $\mu_j\geq 0$ and 
	\begin{equation*}
	\mu_j=k-\sum_{1\leq |\ell|\leq k/2,\,k\neq 0}\cos(\frac{2\pi (j-1)\ell }{n})\leq 2k\quad \forall\, j=1,\dots,n.
	\end{equation*}
	
	Next we establish the bounds for $\min_{2\leq j\leq n}\mu_j$.
	For the upper bound, by the inequality $\cos(\theta)\geq 1-\theta^2/2$ for any $\theta\in\RR$, we have that 
	\begin{equation}
	\min_{2\leq j\leq n}\mu_j\leq  \mu_2 \leq (k+1)- \sum_{\ell=-k/2}^{k/2}(1-\frac{2\pi^2 \ell^2}{n^2})=2\pi^2 \sum_{\ell=-k/2}^{k/2}\frac{\ell^2}{n^2}=\frac{\pi^2 k(k+1)(k+2)}{6 n^2}.
	\end{equation}
	For the lower bound, by applying Lemma \ref{lem:jvpojwqfq} {(see below)} with $h_k(\theta)$ defined as $\sum_{\ell=-k/2}^{k/2}\cos( 2\ell\theta)$ and $\alpha = \frac{\pi}{n}$, we reach
	\begin{align}\label{eqn:vbjoemqf}
	\min_{2\leq j\leq n}\mu_j\geq& \inf_{\theta \in [\alpha,\pi-\alpha]}\left(k+1-h_k(\theta)\right)\nonumber\\
	\geq& k+1-\max\left\{\underbrace{(k+1)\left(1-\frac{k(k+2)\pi^2/n^2}{1+(k+1)^2\pi^2/n^2}\right)^\frac{1}{2}}_{\mathrm{I}},\underbrace{h_k(\pi/n)}_{\mathrm{II}}\right\}
	\end{align}
	For term I, using the inequality  $(1-z)^\frac{1}{2}\leq 1-\frac{1}{2}z$ for any $z\in[0,1]$, and $k+1\leq n$, we have 
	\begin{align}
	\mathrm{I}=&(k+1)\left(1-\frac{k(k+2)\pi^2/n^2}{1+(k+1)^2\pi^2/n^2}\right)^\frac{1}{2}\nonumber\\
	\leq&(k+1)\left(1-\frac{k(k+2)\pi^2/n^2}{2(1+(k+1)^2\pi^2/n^2)}\right)\leq (k+1)\left(1-\frac{\pi^2k(k+2)}{2(1+\pi^2)}\right).\label{eqn:jvoqwdqrdfA}
	\end{align}
	For term II, using the inequality  $\cos(\theta)\leq 1-\theta^2/2+\theta^4/24\leq 1-\frac{12-\pi^2}{24}\theta^2$ for all $\theta \in[0,\pi] $, we have 
	\begin{align}
	\mathrm{II}=&h_k(\pi/n)=\sum_{\ell=-k/2}^{k/2}\cos( \frac{2\pi\ell}{n})\leq \sum_{\ell=-k/2}^{k/2}\left(1-\frac{(12-\pi^2)\pi^2}{6}\frac{\ell^2}{n^2}\right)\nonumber\\
	=&k+1-\frac{(12-\pi^2)\pi^2}{6}\sum_{\ell=-k/2}^{k/2}\frac{\ell^2}{n^2}=k+1-(1-\pi^2/12)\frac{\pi^2 k(k+1)(k+2)}{6 n^2}.\label{eqn:jvoqwdqrdfB}
	\end{align}
	Plugging \eqref{eqn:jvoqwdqrdfA} and \eqref{eqn:jvoqwdqrdfB} into \eqref{eqn:vbjoemqf}, we reach \eqref{eqn:gjfoeqf}.
\end{proof}

We also establish an auxiliary lemma to facilitate the results in Lemma \ref{prop-laplacian}. 
\begin{lemma}\label{lem:jvpojwqfq}
	For any $k\geq 1$ and $k$ is even, let $h_k(\theta)=\sum_{\ell=-k/2}^{k/2}\cos(2\ell \theta)$ for any $\theta\in[\alpha,\pi-\alpha]$ with some $\alpha\in(0,\frac{\pi}{k+1}]$. It holds that 
	\begin{align}\label{lem:fasfsadf}
	\max_{\theta \in [\alpha,\pi-\alpha]}h_k(\theta)\leq\max\left\{
	(k+1)\sqrt{\frac{1+\alpha^2}{1+(k+1)^2\alpha^2}},h_k(\alpha)\right\}
	\end{align}
\end{lemma}
\begin{proof}
	Since $|h_k(\pi/2)|=1$ while $(k+1)\sqrt{\frac{1+\alpha^2}{1+(k+1)^2\alpha^2}} > 1$, we know \eqref{lem:fasfsadf} holds when $\theta = \pi/2$. Since $h_k(\pi-\theta)=h_k(\theta)$ by the definition of $h_k(\theta)$, it suffices to consider $\max_{\theta \in (\alpha,\pi/2)}|h_k(\theta)|$. 
	
	For any $\theta \in(\alpha,\pi/2)$, it holds that 
	\begin{align}
	h_k(\theta)=&\sum_{\ell=-k/2}^{k/2}\cos(2\ell \theta)=\frac{1}{2\sin(\theta)}\sum_{\ell=-k/2}^{k/2}2\cos(2\ell \theta)\sin(\theta)\nonumber\\
	=&\frac{1}{2\sin(\theta)}\sum_{\ell=-k/2}^{k/2}\left(\sin((2\ell+1)\theta)-\sin((2\ell-1)\theta)\right)\nonumber\\
	=&\frac{\sin((k+1)\theta)}{\sin(\theta)}.\nonumber
	\end{align}
	Therefore, we have that 
	\begin{align*}
	h_k^\prime(\theta) =&\frac{(k+1)\cos((k+1)\theta)\sin(\theta)-\sin((k+1)\theta)\cos(\theta)}{\sin^2(\theta)}\\
	=&\frac{\cos((k+1)\theta)\cos(\theta)}{\sin^2(\theta)}\left((k+1)\tan(\theta)-\tan((k+1)\theta)\right).\quad  (\text{if } \cos((k+1)\theta)\cos(\theta)\neq 0)
	\end{align*}
	The special case that $\cos((k+1)\theta)\cos(\theta)= 0$ and $ h_k^\prime(\theta)=0$ leads to that  $\cos(\theta)= 0$,  which  is impossible for $z\in(\alpha, \pi/2)$. Therefore, all the extrema of $h_k$ on $(\alpha,\pi/2)$ must satisfy 
	\begin{equation}\label{eqn:condi-stat}
	\tan((k+1)\theta)=(k+1)\tan(\theta).
	\end{equation} 
	
	By combining the cases of endpoints, i.e., $\theta=\alpha$ and $\theta=\pi/2$, with \eqref{eqn:condi-stat} for the interior extrema, we have that 
	\begin{align}
	&\max_{\theta \in [\alpha,\pi/2]}h_k(\theta)= \max\{\sup_{\substack{\theta \in (\alpha,\pi/2)\nonumber\\ \eqref{eqn:condi-stat}\text{ holds}}}h_k(\theta),h_k(\alpha),h_k(\pi/2)\}\\
	\leq&  \max\{\sup_{\substack{\theta \in (\alpha,\pi/2)\\\ \eqref{eqn:condi-stat}\text{ holds}}}\sqrt{h_k(\theta)^2},h_k(\alpha),1\}\nonumber\\
	= &\max\{\sup_{\substack{\theta \in (\alpha,\pi/2)\nonumber\\ \eqref{eqn:condi-stat}\text{ holds}}}\left(\frac{\frac{\tan((k+1)\theta)^2}{1+\tan((k+1)\theta)^2}}{\frac{\tan(\theta)^2}{1+\tan(\theta)^2}}\right)^\frac{1}{2},h_k(\alpha)^2,1\}\\
	\leq& \max\{\sup_{\substack{\theta \in (\alpha,\pi/2)\\\ \eqref{eqn:condi-stat}\text{ holds}}}\left(\frac{\frac{(k+1)^2\tan(\theta)^2}{1+(k+1)^2\tan(\theta)^2}}{\frac{\tan(\theta)^2}{1+\tan(\theta)^2}}\right)^\frac{1}{2},h_k(\alpha),1\} \nonumber\\
	\leq &\max\{(k+1) \sup_{\theta \in (\alpha,\pi/2)}\left(\frac{1+\tan(\theta)^2}{1+(k+1)^2\tan(\theta)^2}\right)^\frac{1}{2},h_k(\alpha)\}\nonumber\\
	\leq& \max\{ (k+1)\left(\frac{1+\alpha^2}{1+(k+1)^2\alpha^2}\right)^\frac{1}{2},h_k(\alpha)\},\nonumber
	\end{align}
	where the last inequality is because the function $\frac{1+z}{1+(k+1)^2z}$ is decreasing with respect to $z\in[0,+\infty)$ and $\tan(\theta)\geq \tan(\alpha)\geq \alpha$. 
\end{proof}

\textbf{Theorem \ref{thm-D-beta-relation}}  \textit{Given a fixed $n$ and any $\beta \in [0,\cos(\pi/n)]$, there exists a ring-lattice graph with an associated weight matrix $W$ such that
	\begin{itemize}
		\item[(i)]  $W\in\cW_{n,\beta}$ , i.e., $W \in \RR^{n\times n}$ and $\|W-\frac{1}{n}\mathds{1}_n\mathds{1}_n^T\|=\beta$;
		\item[(ii)] its diameter $D$ satisfies $D=\Theta(1/\sqrt{1-\beta})$.
	\end{itemize}
}
\begin{proof}
We prove this theorem in two cases: 

\emph{(Case 1: $0\leq \beta\leq\min\{ \cos(\pi/9),\cos(\pi/n)\}$.)} In this case, we consider the complete graph whose diameter $D=1$, and let $W=\frac{1-\beta}{n}\mathds{1}_n\mathds{1}_n^T+\beta I$, then it holds that $\|W-\frac{1}{n}\mathds{1}_n\mathds{1}_n^T\|_2=\beta$. Furthermore, since $1\geq 1-\beta\geq 1-\cos(\pi/9)=\Omega(1)$, we naturally have $1 =D=\Omega((1-\beta)^{-\frac{1}{2}})$. Note that the complete graph can also be regarded as a special ring-lattice graph $\cR_{n,n-1}$. 

\emph{(Case 2: $\cos(\pi/9)< \beta \leq \cos(\pi/n)$.)} Note that this case requires $n\geq 10$. Letting 
\begin{equation}\label{eqn:k-choice}
k=2\left\lceil n\sqrt{\frac{3(1-\beta)}{\pi^2(1-\pi^2/12)}}\right\rceil\geq \max\left\{2 n\sqrt{\frac{3(1-\beta)}{\pi^2(1-\pi^2/12)}},2\right\},
\end{equation} 
we consider the ring-lattice graph $\cR_{n,k}$ with the weight matrix 
\begin{align}\label{W-defi}
W=I-\frac{1-\beta}{\min_{2\leq j\leq n}\mu_j}L_{n,k}
\end{align}
where $k$ is defined as in \eqref{eqn:k-choice} and $\{\mu_j:2\leq j\leq n\}$ are given in Lemma \ref{prop-laplacian}. For such ring-lattice graph and its associated weight matrix defined in \eqref{W-defi}, we prove $D = \Omega((1-\beta)^{-\frac{1}{2}})$ in following steps:
\begin{itemize}[leftmargin=1em]
\item We first check whether $\cR_{n,k}$ is well-defined, i.e., $2\leq k<n-1$ and $k$ is even. It is easy to observe that $k$ is even by \eqref{eqn:k-choice}. Furthermore, since $\cos(\pi/9)<\beta$ and $n\geq 10$, we have 
\begin{align}\label{eqn:k-choice2}
k&<2\left( n\sqrt{\frac{3(1-\beta)}{\pi^2(1-\pi^2/12)}}+1\right)<2\left( n\sqrt{\frac{3(1-\cos(\pi/9))}{\pi^2(1-\pi^2/12)}}+1\right)<0.7n+2\leq n-1.
\end{align}

\item We next check whether $W_{n,k}\in\cW_{n,\beta}$, i.e., $\|W-\frac{1}{n}\mathds{1}_n\mathds{1}_n^T\|_2=\beta$. By Lemma \ref{prop-laplacian}, we have
\begin{equation*}
\frac{1-\beta}{\min_{2\leq j\leq n}\mu_j}\leq \frac{6(1-\beta)n^2}{(1-\pi^2/12)\pi^2k(k+1)(k+2)}<\frac{1}{2k}\frac{12(1-\beta)n^2}{(1-\pi^2/12)\pi^2k^2}\leq \frac{1}{2k}.
\end{equation*}
Since all eigenvalues of $L_{n,k}$ lie in $[0,2k]$ (see \eqref{znbznzn1}), it holds that $W$ is positive semi-definite (see definition \eqref{W-defi}). Moreover, 
\begin{equation*}
\|W-\frac{1}{n}\mathds{1}_n\mathds{1}_n^T\|_2 =\max_{2\leq \ell \leq n}\left\{1-\frac{1-\beta}{\min\{\mu_j:2\leq j\leq n\}}\mu_\ell\right\}=\beta\quad i.e.,\quad W\in\cW_{n,\beta}.
\end{equation*}

\item Finally, we check whether $\Theta(\frac{n}{k})=D=\Theta((1-\beta)^{-\frac{1}{2}})$.
By using the inequality $\cos(\theta)\leq 1-\theta^2/2+\theta^4/24$ with $\theta =\pi/n$, we have
\[
\beta\leq \cos(\pi/n)\leq 1-\pi^2/2n^2+\pi^4/24n^4\leq 1-\pi^2/3n^2.
\]
Therefore, we have  that $ n\sqrt{\frac{3(1-\beta)}{\pi^2(1-\pi^2/12)}}\geq n\sqrt{\frac{\pi^2}{n^2(1-\pi^2/12)}}\geq 2$
and hence
\begin{align}\label{eqn:k-choice3}
k<2\left( n\sqrt{\frac{3(1-\beta)}{\pi^2(1-\pi^2/12)}}+1\right)\leq 3n\sqrt{\frac{3(1-\beta)}{\pi^2(1-\pi^2/12)}}.
\end{align}
Combining \eqref{eqn:k-choice} and \eqref{eqn:k-choice3}, we achieve
\begin{equation*}
\frac{n}{k}=\Theta((1-\beta)^{-\frac{1}{2}}). 
\end{equation*}
This fact together with Lemma \ref{lm-diameter} leads to $D = \Theta(1/\sqrt{1-\beta})$. 
\end{itemize}
\end{proof}

\section{Lower bounds}
\label{app-lower-bounds}
\subsection{Non-convex Case: Proof of Theorem \ref{thm-lower-bound-nc}}\label{app:lower-nonconvex}
In this subsection, we establish the lower bound of the smooth and non-convex decentralized stochastic
optimization. Our analysis builds upon \cite{lu2021optimal} but utilizes the proposed ring-lattice graphs for the construction of worst-case instances, which significantly broadens the scope of weight matrices that the lower bound can apply to, i.e., from $\beta=\cos(\pi/n)$ in \cite{lu2021optimal} to any $\beta\in[0,\cos(\pi/n)]$. Specifically, we will prove for any  $\beta\in[0,\cos(\pi/n)]$ and $T=\Omega(1/\sqrt{1-\beta})$,
\begin{align}\label{eqn:fjowfqrgfsa}
    &\inf_{A\in\cA_W}\sup_{W\in\cW_{n,\beta}} \sup_{\{\tilde{g}_i\}_{i=1}^n\subseteq\cO_\sigma^2}\sup_{\{f_i\}_{i=1}^n\subseteq \cF_{L}}\mathbb{E}\|\nabla f(\hat{x}_{A, \{f_i\}_{i=1}^n,\{\tilde{g}_i\}_{i=1}^n,W,T})\|^2 \nonumber\\
    =&   \Omega\left(\frac{\sqrt{\Delta L}\sigma}{\sqrt{nT}}+ \frac{\Delta L}{T\sqrt{1-\beta}}\right).
\end{align}
where $\Delta := \mathbb{E}[f(x^{(0)})] -\min_{x}f(x)$.

We establish the two terms in \eqref{eqn:fjowfqrgfsa} separately as in  \cite{lu2021optimal} by constructing two hard-to-optimize instances.
We denote the $j$-th coordinate of a vector $x\in\RR^d$ by $[x]_j$ for $j=1,\dots,d$, and 
 let $\prog(x)$ be 
\begin{equation*}
    \prog(x):=\begin{cases}
    0 & \text{if $x=0$};\\
    \max_{1\leq j\leq d}\{j:[x]_j\neq 0\}& \text{otherwise}.
    \end{cases}
\end{equation*}
Similarly, for a set of multiple points $\cX=\{x_1,x_2,\dots\}$, we define $\prog(\cX):=\max_{x\in\cX}\prog(x)$.
As described in \cite{carmon2020lower,Arjevani2019LowerBF}, a  function $f$ is called zero-chain if it satisfies
\begin{equation*}
    \prog(\nabla f(x))\leq \prog(x)+1,\quad\forall\,x\in\RR^d,
\end{equation*}
which implies that, starting from $x=0$, a single gradient evaluation can only earn at most one more non-zero coordinate for the model parameters. We next introduce a key zero-chain function to facilitate the analysis.
\begin{lemma}[Lemma 2 of \cite{Arjevani2019LowerBF}]\label{lem:basic-fun}
Let function 
\begin{equation*}
    \ell(x):=-\Psi(1) \Phi([x]_{1})+\sum_{j=1}^{d-1}\Big(\Psi(-[x]_j) \Phi(-[x]_{j+1})-\Psi([x]_j) \Phi([x]_{j+1})\Big)
\end{equation*}
where for $\forall\, z \in \mathbb{R},$
$$
\Psi(z)=\begin{cases}
0 & z \leq 1 / 2; \\
\exp \left(1-\frac{1}{(2 z-1)^{2}}\right) & z>1 / 2,
\end{cases} \quad \Phi(z)=\sqrt{e} \int_{-\infty}^{z} e^{-\frac{1}{2} t^{2}} \mathrm{d}t.
$$
Then $\ell$ satisfy several properties as below:
\begin{enumerate}
    \item $\ell(x)-\inf_{x} \ell(x)\leq \Delta_0 d$, $\forall\,x\in\RR^d$ with $\Delta_0=12$.
    \item $\ell$ is $L_0$-smooth with $L_0=152$.
    \item $\|\nabla \ell(x)\|_\infty\leq G_0 $, $\forall\,x\in\RR^d$ with $G_0= 23$.
    \item $\|\nabla \ell(x)\|_\infty\ge 1 $ for any $x\in\RR^d$ with $[x]_d=0$. 
\end{enumerate}
\end{lemma}

We next establish the two terms in \eqref{eqn:fjowfqrgfsa} by constructing two hard-to-optimize instances.

\textbf{Instance 1.} The proof for the first term  $\Omega((\frac{\Delta L\sigma^2}{nT})^\frac{1}{2})$ essentially follows \cite{lu2021optimal}. We provide the proof for the sake of being self-contained.

(Step 1.) Let all $f_i$ be $L\lambda^2\ell(x/\lambda)/L_0$ where $\ell$ is defined in Lemma \ref{lem:basic-fun} and $\lambda$ is to be specified, and thus $f=L\lambda^2\ell(x/\lambda)/L_0$. Since $\nabla^2 f_i=L\nabla^2 \ell /L_0$ and $h$ is $L_0$-smooth (Lemma \ref{lem:basic-fun}), we know $f_i$ is $L$-smooth for any $\lambda>0$.
By Lemma \ref{lem:basic-fun}, we have
\begin{equation*}
    f(0)-\inf_x f(x)=\frac{L\lambda^2}{L_0}(\ell(0)-\inf_x \ell(x)) {\leq}\frac{L\lambda^2\Delta_0d}{L_0}. 
\end{equation*}
Therefore, to ensure $f_i\in\cF_{L}$ for all $1\leq i\leq n$ and $f(0)-\inf_{x}f(x)\leq \Delta$, it suffices to let 
\begin{equation}\label{eqn:jgowemw}
    \frac{L\lambda^2\Delta_0d}{L_0}\leq \Delta, \quad \text{i.e.,}\quad d\lambda^2\leq \frac{L_0 \Delta}{L\Delta_0}.
\end{equation}

(Step 2.) We construct the stochastic gradient oracle $\tilde{g}_i$ $\forall\,i=1,\dots,n$ as the follows:
\begin{equation*}
    [\tilde{g}_i(x)]_j=[\nabla f_i(x)]_j\left(1+\mathds{1}{\{j>\prog(x)\}}\left(\frac{Z}{p}-1\right)\right), \forall\,x\in\RR^d,\,j=1,\dots,d
\end{equation*}
with  $Z\sim \text{Bernoulli}(p)$, and $p\in(0,1)$ to be specified. The oracle $\tilde{g}_i(x)$ has probability $p$ to zero  $[\nabla f_i(x)]_{\prog(x)+1}$. 
It is easy to see $\tilde{g}_i$ is unbiased, i.e., $\EE[\tilde{g}_i(x)]=\nabla f_i(x)$ for all $x\in\RR^d$. Moreover, since $f_i$s are zero-chain, we have $\prog(\tilde{g}_i(x))\leq \prog(\nabla f_i(x))\leq \prog(x)+1$ and hence
\begin{align*}
\EE[\|\tilde{g}_i(x)-\nabla f_i(x)\|^2]&=|[\nabla f_i(x)]_{\prog(x)+1}|^2\EE\left[\left(\frac{Z}{p}-1\right)^2\right]=|[\nabla f_i(x)]_{\prog(x)+1}|^2\frac{1-p}{p}\\
&\leq \|\nabla f_i(x)\|_\infty^2\frac{1-p}{p}
\leq \frac{L^2\lambda^2(1-p)}{L_0^2p}\|\nabla \ell(x)\|_\infty^2\\
&\overset{\text{Lemma \ref{lem:basic-fun}}}{\leq } \frac{L^2\lambda^2(1-p)G_0^2}{L_0^2p}.
\end{align*}
Therefore, to ensure $\tilde{g}_i\in\cO_{\sigma^2}$ for all $1\leq i\leq n$, it suffices to let 
\begin{equation}\label{eqn:gjvownefoq}
     p=\min\{\frac{L^2\lambda^2G_0^2}{L_0^2\sigma^2},1\}.
\end{equation}
We let $W=(1-\beta)\frac{1}{n}\one_n\one_n^T+\beta I$, then obviously $W\in\cW_{n,\beta}$. 

(Step 3.) Next we show the error $\EE[\|\nabla f(x)\|^2]$ is lower bounded by $\Omega((\frac{\Delta L\sigma^2}{nT})^\frac{1}{2})$, with any algorithm $A\in\cA_W$. Let $x^{(t)}_i$, $\forall\,t=0,\dots$ and $1\leq i\leq n$, be the $t$-th query point of node $i$. Let $\prog^{(t)}=\max_{1\leq i\leq n,\,0\leq s< t}\prog(x^{(s)}_i)$. By  Lemma 2 of \cite{lu2021optimal}, we have 
\begin{equation}\label{eqn:vjowemnfq}
    \PP(\prog^{(T)}\geq d)\leq e^{(e-1)npT-d}.
\end{equation}
On the other hand, when $\prog^{(T)}< d$, by the fourth point in Lemma \ref{lem:basic-fun}, it holds that 
\begin{align}\label{eqn:vjowemnfq2}
    \min_{\hat{x}\in\mathrm{span}\{\{x_i^{(t)}\}_{1\leq i\leq n,0\leq t< T}\}}\|\nabla f(\hat{x})\|\geq& \min_{[\hat{x}]_{d}=0}\|\nabla f(\hat{x})\|=\frac{L\lambda}{L_0}\min_{[\hat{x}]_{d}=0}\|\nabla \ell(\hat{x})\|\geq \frac{L\lambda}{L_0}.
\end{align}                                                       
Therefore, by combining \eqref{eqn:vjowemnfq} and \eqref{eqn:vjowemnfq2}, we have
\begin{equation}\label{eqn:Lvjowenfq}
    \EE[\|\nabla f(\hat{x})\|^2]\geq (1-e^{(e-1)npT-d})\frac{L^2\lambda^2}{L_0^2}.
\end{equation}
Let 
\begin{equation}\label{eqn:jvbomewfqwgswdg}
    \lambda=\frac{L_0}{L}\left(\frac{\Delta L\sigma^2}{3nT  L_0\Delta_0 G_0^2}\right)^\frac{1}{4}\quad \text{and}\quad d=\left\lfloor\left(\frac{3L\Delta nT G_0^2}{\sigma^2L_0\Delta_0}\right)^\frac{1}{2}\right\rfloor.
\end{equation}
Then \eqref{eqn:jgowemw} naturally holds and $p=\min\{\frac{G_0^2}{\sigma^2}\left(\frac{\Delta L\sigma^2}{3nT L_0\Delta_0 G_0^2}\right)^\frac{1}{2},1\}$ by plugging \eqref{eqn:jvbomewfqwgswdg}  into \eqref{eqn:gjvownefoq}. Without loss of generality, we assume $T$ is sufficiently large such that $d\geq 2$. Then, using the definition of $p$, we have that
\begin{align}
    &(e-1)npT-d\leq (e-1)nT\; \frac{G_0^2}{\sigma^2}\left(\frac{\Delta L\sigma^2}{3nTL_0\Delta_0 G_0^2}\right)^\frac{1}{2}-d\nonumber\\
    =&\frac{e-1}{3}\left(\frac{3 L\Delta nT G_0^2}{\sigma^2 L_0\Delta_0 }\right)^\frac{1}{2}-d< \frac{e-1}{3}(d+1)-d\leq  1-e<0\label{eqn:Lvjobqwm}\nonumber
\end{align}
which, combined with \eqref{eqn:Lvjowenfq}, further implies
\begin{equation*}
    \EE[\|\nabla f(\hat{x})\|^2]=\Omega\left(\frac{L^2\lambda^2}{L_0^2}\right)=\Omega\left(\left(\frac{\Delta L\sigma^2}{3nT  L_0\Delta_0 G_0^2}\right)^\frac{1}{2}\right)=\Omega\left(\left(\frac{\Delta L\sigma^2}{nT}\right)^\frac{1}{2}\right).
\end{equation*}

\paragraph{Instance 2.} The proof for the second term  $\Omega(\frac{\Delta L}{T\sqrt{1-\beta}})$ utilizes weight matrices defined on the ring-lattice graphs  described in Theorem \ref{thm-D-beta-relation}. 

(Step 1.) Let functions 
\begin{equation*}
    \ell_1(x):=-\frac{n}{\lceil n/3\rceil}\Psi(1) \Phi([x]_{1})+\frac{n}{\lceil n/3\rceil}\sum_{j \text{ even, } 0< j<d}\Big(\Psi(-[x]_j) \Phi(-[x]_{j+1})-\Psi([x]_j) \Phi([x]_{j+1})\Big)
\end{equation*}
and 
\begin{equation*}
    \ell_2(x):=\frac{n}{\lceil n/3\rceil}\sum_{j \text{ odd, } 0<j<d}\Big(\Psi(-[x]_j) \Phi(-[x]_{j+1})-\Psi([x]_j) \Phi([x]_{j+1})\Big).
\end{equation*}
Compared to the $\ell$ function in instance 1,  the 
$\ell_1$ and $\ell_2$ defined here are $3L_0$-smooth.
Furthermore, let
\begin{equation*}
    f_i=\begin{cases}
    L\lambda^2 \ell_1(x/\lambda)/(3L_0)&\text{if }i\in E_1\triangleq\{j:1\leq j\leq \lceil \frac{n}{3}\rceil\},\\
    L\lambda^2 \ell_2(x/\lambda)/(3L_0)&\text{if }i\in E_2\triangleq\{j:\lfloor \frac{n}{2}\rfloor +1\leq j\leq \lfloor\frac{n}{2}\rfloor+\lceil \frac{n}{3}\rceil\},\\
    0&\text{else.}
    \end{cases}
\end{equation*}
where $\lambda>0$ is to be specified.
To ensure $f_i\in\cF_{L}$ for all $1\leq i\leq n$ and $f(0)-\inf_{x}f(x)\leq \Delta$, it suffices to let
\begin{equation}\label{eqn:jgowemw-dfshadasd}
    \frac{L\lambda^2\Delta_0d}{3L_0}\leq \Delta, \quad \text{i.e.,}\quad d\lambda^2\leq \frac{3L_0 \Delta}{L\Delta_0}.
\end{equation}
With the functions defined above, we have  $f(x)=\frac{1}{n}\sum_{i=1}^n f_i(x)=L\lambda^2 \ell(x/\lambda)/(3L_0)$ and 
\begin{align*}
    \prog(\nabla f_i(x))
    \begin{cases}
    =\prog(x) +1&\text{if } \{\prog(x) \text{ is even and } i\in E_1\}\cup\{\prog(x) \text{ is odd and }i\in E_2\}\\
    \leq \prog(x) &\text{otherwise}.
    \end{cases}
\end{align*}
Therefore, to make progress (i.e., to increase $\prog(x)$), for any gossip algorithm $A\in\cA_{W}$, one must take the gossip communication protocol to transmit information between $E_1$ to $E_2$ alternatively. Namely, it takes at least $\dist(E_1,E_2)$ rounds of gossip communications for any possible gossip algorithm $A$ to increase $\prog(\hat{x})$ by $1$. Therefore, we have
\begin{equation}\label{eqn:jofqsfdq}
    \prog^{(T)}=\max_{1\leq i\leq n,\,0\leq t< T}\prog(x^{(t)}_i)\leq \left\lfloor \frac{T}{\dist(E_1,E_2)}\right\rfloor+1,\quad \forall\,T\geq 0.
\end{equation}

(Step 2.) 
We consider a gradient oracle that return lossless full-batch gradients, i.e., $\tilde{g}_i=\nabla f_i(x)$, $\forall\,x\in\RR^d, \,1\leq i\leq n$. For the construction of weight matrix, we consider the ring-lattice graph  with diameter $D$ in Theorem \ref{thm-D-beta-relation} and its associated weight matrix $W$ such that $W\in\cW_{n,\beta}$. Then by Lemma \ref{lm-diameter} and Theorem \ref{thm-D-beta-relation}, we have $\dist(E_1,E_2)=\dist(\lceil {n}/{3}\rceil, \lfloor{n}/{2}\rfloor+1)=\Theta(D)=\Theta(1/\sqrt{1-\beta})$. Suppose $\dist(E_1,E_2)\geq 1 /(C\sqrt{1-\beta})$ with some absolute constant $C$, then by \eqref{eqn:jofqsfdq}, we have
\begin{equation}\label{eqn:jofqsfdq-0gsfa}
    \prog^{(T)}=\max_{1\leq i\leq n,\,0\leq s\leq T}\prog(x^{(s)}_i)\leq \left\lfloor {C\sqrt{1-\beta} T}\right\rfloor+1,\quad \forall\,T\geq 0.
\end{equation}

(Step 3.) We finally show the error $\EE[\|\nabla f(x)\|^2]$ is lower bounded by $\Omega\left(\frac{\Delta L}{\sqrt{1-\beta}T}\right)$, with any algorithm $A\in\cA_{W}$. 
For any $T\geq 1/(C\sqrt{1-\beta})=\Omega(1/\sqrt{1-\beta})$, consider
\begin{equation*}
    d= \left\lfloor C\sqrt{1-\beta} T\right\rfloor+2 <4C\sqrt{1-\beta} T
\end{equation*}
and
\begin{equation}\label{eqn:gjofefqfq}
     \lambda =\left(\frac{3L_0 \Delta}{4L \Delta_0 C\sqrt{1-\beta}T}\right)^\frac{1}{2} .
\end{equation}
Then \eqref{eqn:jgowemw-dfshadasd} naturally holds.
Since $\prog^{(T)}<d$ by \eqref{eqn:jofqsfdq-0gsfa}, following \eqref{eqn:vjowemnfq2} and using \eqref{eqn:gjofefqfq}, we have 
\begin{equation*}
    \EE[\|\nabla f(\hat{x})\|^2]\geq\min_{[\hat{x}]_{d}=0}\|\nabla f(\hat{x})\|^2\geq \frac{L^2\lambda^2}{9L_0^2}=\Omega\left(\frac{\Delta L}{\sqrt{1-\beta}T}\right).
\end{equation*}

\subsection{Non-convex Case with PL Condition: Proof of Theorem \ref{thm-lower-bound-pl}}
In this subsection, we provide the proof for smooth and non-convex decentralized stochastic
optimization under the PL condition: for any  $\beta\in[0,\cos(\pi/n)]$ and $T=\Omega(1/\sqrt{1-\beta})$,
\begin{align}\label{eqn:fjowfqrgfsa-fasfasfa}
&\inf_{A\in\cA_W}\sup_{W\in\cW_{n,\beta}} \sup_{\{\tilde{g}_i\}_{i=1}^n\subseteq\cO_\sigma^2}\sup_{\{f_i\}_{i=1}^n\subseteq \cF_{L,\mu}} \mathbb{E}[f(\hat{x}_{A, \{f_i\}_{i=1}^n,\{\tilde{g}_i\}_{i=1}^n,W,T}) - f^\star]\nonumber\\
= &   \Omega\left(\frac{\sigma^2}{\mu nT}+\frac{\mu \Delta}{L}\exp(-\sqrt{{\mu}/{L}}\sqrt{1-\beta}T)\right).
\end{align}
where $\Delta := \mathbb{E}[f(x^{(0)})]-\min_{x}f(x)$. We still prove the two terms in \eqref{eqn:fjowfqrgfsa-fasfasfa} separately. 

\textbf{Instance 1.}
Our proof for first term $\Omega(\frac{\sigma^2}{\mu nT} )$ is inspired by \cite{pmlr-v28-ramdas13}, who study the lower bounds in the stochastic but single-node regime. 

We consider all functions $f_i=f$ are homogeneous and  $W=\frac{1-\beta}{n}\mathds{1}_n\mathds{1}_n^T+\beta I\in \cW_{n\beta}$. We then choose two functions $f^1,\,f^{-1}\in\cF_{L,\mu}$ which are close enough to each other so that $f^1$ and $f^{-1}$ are hard to distinguish within $T$  gradient queries on each node. The indistinguishability between $f^1$ and $f^{-1}$ follows the standard Le Cam's method in hypothesis testing.
Next, we carefully show that the indistinguishability between $f^1$ and $f^{-1}$ can be properly translated  into the lower bound of the algorithmic performance.

Recall that $[x]_j$
denotes the $j$-th coordinate of vector $x\in\mathbb{R}^d$. 
Let 
\[
f^{v}= \frac{1}{2}\left(\mu([x]_1-v\lambda )^2+L \sum_{j=2}^d[x]_j^2\right),
\]
where $v\in\{\pm 1\}$, $d$ can be any integer greater than $2$, and $\lambda>0$ is a quantity to be determined. Clearly, $f^v$ is $L$-smooth and $\mu$-strongly convex, and thus satisfies the $\mu$-PL condition.
The optimum $x^{v,\star}$ of $f^v$ is $v\lambda e_1$ with function value $\min_{x\in\RR^d} f^{v}(x)=0$, where $e_1$ is the first canonical vector.

We construct the gradient oracle with $\cN(0,\sigma^2)$ noise.
Specifically, on query at point $x$ on  node $i$, the oracle returns $\tilde{g}_i(x)=\nabla f_i(x)+s_i$ with independent noise $s_i\sim \cN(0,\sigma^2)$. Obviously $\tilde{g}_i\in\cO_{\sigma^2}$ for all $1\leq i\leq n$.

Let $\vS^{(T)}:=\{(\vx^{(t)}\triangleq(x^{(t)}_1,\dots,x^{(t)}_n),\tilde{\vg}(\vx^{(t)})\triangleq (\tilde{g}_1(x^{(t)}_1),\dots,\tilde{g}_n(x^{(t)}_n))\}_{t=0}^{T-1}$ be set of variables corresponding to the sequence of $T$ queries on all nodes. Let $P^{v,T}:=P(\vS^{(T)}\mid f^v)$ be the joint distribution of  $\vS^{(T)}$ if the underlying function was $f^v$, i.e., $f_1=\cdots=f_n=f^v$.

\begin{lemma}\label{lem:vjoqwdsa}
Let $V\sim \mathrm{Unif}(\{\pm 1\})$, then for any optimization procedure $\hat{x}$ based on the observed gradients, it holds that 
\begin{align}
    \EE_{V}[f^V(\hat{x})]\geq \mu \lambda^2/2 \inf_{\hat{V}} \PP(\hat{V}\neq V)
\end{align}
where the randomness is over the random index $V$ and the observed data $\vS^{(T)}$ and the infimum is taken over all testing procedures $\hat{V}$ based on the observed queries.
\end{lemma}
\begin{proof}
Since $P^{v,T}$ is probability of the observed queries conditioned on $\{V=v\}$, by Markov's inequality, we have 
\begin{align}\label{eqn:vgjoweni}
    \EE_{V}[f^V(\hat{x})]\geq \mu \lambda^2/2\, \EE_{V}[P^{V,T}(f^V(\hat{x})\geq \mu \lambda^2/2)].
\end{align}
Now, we define the test $\hat{V}_{\mathrm{op}}$ based on the optimization procedure $\hat{x}$ as follows:
\begin{align*}
    \hat{V}_{\mathrm{op}}=\begin{cases}
    v &\text{if $f^v(\hat{x})< \mu \lambda^2/2$}\\
    \text{randomly pick one from $\{\pm1\}$} & \text{otherwise.} 
    \end{cases},
\end{align*}
By the definition of $f^v$, at most one of $\{f^{1}(\hat{x}),f^{-1}(\hat{x})\}$ is strictly below $\mu \lambda^2/2$, so $\hat{V}_{\mathrm{op}}$ is well-defined. 
By the definition of $\hat{V}_{\mathrm{op}}$, $\{\hat{V}_{\mathrm{op}}\neq v\}$ implies $\{f^{v}(\hat{x})\geq \mu \lambda^2/2\}$. Thus we have 
\begin{align}\label{eqn:gvkopwemfe}
   P^{v,T}(f^v(\hat{x})\geq \mu \lambda^2/2)\geq P^{V,T}(\hat{V}_{\mathrm{op}}\neq v)\geq \inf_{\hat{V}} \PP(\hat{V}\neq v)\quad\text{for any }v\in\{\pm1\}
\end{align}
Plugging \eqref{eqn:gvkopwemfe} into \eqref{eqn:vgjoweni}, we reach the conclusion.
\end{proof}

Based on Lemma \ref{lem:vjoqwdsa}, we know that the worst-case optimization performance in the above scenario can be lower bounded by the indistinguishabiligty between to hypothesises $f^1$ and $f^{-1}$. Recall the following standard result of Le Cam \cite{Cam1986AsymptoticMI}
\begin{align}
    \inf_{\hat{V}} \left(P^{1,T }(\hat{V} \neq 1)+P^{-1,T }(\hat{V}\neq -1)\right)=1- \|P^{1,T}-P^{-1,T}\|_{TV}
\end{align}
where $\|P^{1,T}-P^{-1,T}\|_{TV}$ is the total variation distance between the two distributions.  We can precisely measure the indistinguishabiligty between  $f^1$ and $f^{-1}$ by $\|P^{1,T}-P^{-1,T}\|_{TV}$. Since the total variation distance is hard to compute, we turn to compute the KL-divergence, which is a relaxation of the total variation distance due to Pinsker's inequality \cite{Csiszr2011InformationT}:
\begin{align}
    \|P^{1,T}-P^{-1,T}\|_{TV}^2\leq \frac{1}{2} \mathrm{D_{KL}}(P^{1,T}||P^{-1,T}).
\end{align}
\begin{lemma}\label{lem:kldiv}
The KL-divergence between $P^{1,T}$ and $P^{-1,T}$ is upper bounded by:
\begin{align}\label{eqn:kldiv}
    \mathrm{D_{KL}}(P^{1,T}||P^{-1,T})\leq 2nT\mu^2\lambda^2/\sigma^2.
\end{align}
\end{lemma}
\begin{proof}
Since the observed data obey Markov's property, i.e., 
$\tilde{g}_i(x^{(t)}_i)\perp \vS^{(t)} \mid x^{(t)}_i$ for any $1\leq i\leq n$ and $0\leq t<T$, we have
\begin{align*}
    &\mathrm{D_{KL}}(P^{1,T}||P^{-1,T})=\EE_{P^{1,T}}\left[\log\left(P^{1,T}/P^{-1,T}\right)\right]\\
    =&\EE_{P^{1,T}}\left[\log\left(\prod_{t=0}^{T-1}\PP(\tilde{\vg}(\vx^{(t)})\mid \vx^{(t)},f^1)/\prod_{t=0}^{T-1}\PP(\tilde{\vg}(\vx^{(t)})\mid \vx^{(t)},f^{-1} )\right)\right]\\
    =&\sum_{t=0}^{T-1}\EE_{\vS^{(t)}\mid f^{1}}\left[\EE_{\tilde{\vg}(\vx^{(t)})\sim\cN(\nabla f^1(\bx),\sigma^2)}\left[\log\left(\PP(\vg(\vx^{(t)})\mid \vx^{(t)},f^1 )/\PP(\vg(\vx^{(t)})\mid \vx^{(t)},f^{-1} )\right)\right]\right]\\
    =&\sum_{t=0}^{T-1}\sum\limits_{i=1}^n\EE_{\vS^{(t)}\mid f^{1}}\left[\mathrm{D_{KL}}\left(\PP(\tilde{g}(x_i^{(t)})\mid x_i^{(t)},f^1 )||\PP(g(x_i^{(t)})\mid x_i^{(t)},f^{-1} )\right)\right].
\end{align*}
Then we apply the uniform upper bound over $x$ as follows:
\begin{align}\label{eqn:vjoqwnesc}
    \mathrm{D_{KL}}(P^{1,T}||P^{-1,T})
    &\leq  nT\sup_{x,i}\mathrm{D_{KL}}\left(\PP(\tilde{g}_i(x)\mid x,f^1 )||\PP(\tilde{g}_i(x)\mid x,f^{-1} )\right).
\end{align}
Note the KL-divergence between two Gaussians can be computed explicitly from their means and variances as follows: for any $x\in\RR^d$
\begin{align}\label{eqn:gvi0qwe}
    &\mathrm{D_{KL}}\left(\PP(\tilde{g}_i(x)\mid x,f^1 )||\PP(\tilde{g}_i(x)\mid x,f^{-1} )\right)\nonumber\\
    =&\mathrm{D_{KL}}\left(\cN(\nabla f^{1}(x),\sigma^2))||\cN(\nabla f^{-1}(x),\sigma^2))\right)
    =\frac{1}{2\sigma^2}\|\nabla f^1(x)-\nabla f^{-1}(x)\|^2= \frac{2\mu^2\lambda^2}{\sigma^2}.
\end{align}
Plugging \eqref{eqn:gvi0qwe} into \eqref{eqn:vjoqwnesc}, we reach the result \eqref{eqn:kldiv}.
\end{proof}

Therefore, by Lemma \ref{lem:vjoqwdsa} and Lemma \ref{lem:kldiv}, we can obtain the 
lower bound based \eqref{eqn:fjowfqrgfsa-fasfasfa}. Specifically, for $f^1,\,f^{-1}\in\cF_{L,\mu}$ constructed above, assume that the algorithm $A$ receives stochastic gradients with $\cN(0,\sigma^2)$ noise, then we have
\begin{align}
    &\inf_{A\in\cA_W}\sup_{W\in\cW_{n,\beta}} \sup_{\{\tilde{g}_i\}_{i=1}^n\subseteq\cO_\sigma^2}\sup_{\{f_i\}_{i=1}^n\subseteq \cF_{L,\mu}} \mathbb{E}[f(\hat{x}_{A, \{f_i\}_{i=1}^n,\{\tilde{g}_i\}_{i=1}^n,W,T}) - f^\star]\nonumber\\
    \geq &\inf_{A\in\cA_{W}}\max_{v\in\{\pm 1\}}\mathbb{E}[f^v(\hat{x}_{A, \{f^v\}_{i=1}^n,\cN(\nabla f^v(x),\sigma^2),\frac{1-\beta}{n}\mathds{1}_n\mathds{1}_n^T+\beta I,T})]\nonumber\\
    \geq &\inf_{A\in\cA_{W}}\EE_{V\sim \mathrm{Unif}(\{\pm 1\})}[\mathbb{E}[f^V(\hat{x}_{A, \{f^V\}_{i=1}^n,\cN(\nabla f^V(x),\sigma^2),\frac{1-\beta}{n}\mathds{1}_n\mathds{1}_n^T+\beta I,T})]].\label{eqn:vjoqnwds}
\end{align}
Plugging Lemma \ref{lem:vjoqwdsa} and Lemma \ref{lem:kldiv} into \eqref{eqn:vjoqnwds} and using Pinsker's inequality, we immediately have
\begin{align}\label{eqn:gbhieohfgqw}
    &\inf_{A\in\cA_W}\sup_{W\in\cW_{n,\beta}} \sup_{\{\tilde{g}_i\}_{i=1}^n\subseteq\cO_\sigma^2}\sup_{\{f_i\}_{i=1}^n\subseteq \cF_{L,\mu}} \mathbb{E}[f(\hat{x}_{A, \{f_i\}_{i=1}^n,\{\tilde{g}_i\}_{i=1}^n,W,T}) - f^\star]\nonumber\\
    \geq &\frac{\mu \lambda^2}{2}\left(1-\sqrt{\mathrm{D_{KL}}(P^{1,T}||P^{-1,T})/2}\right)\nonumber\\
    =&\frac{\mu \lambda^2}{2}\left(1-\frac{\mu \lambda}{\sigma}\sqrt{nT}\right).
\end{align}
Choosing $\lambda = \frac{2\sigma }{3\mu\sqrt{n T}}$ in \eqref{eqn:gbhieohfgqw}, we reach the lower bound $\Omega(\frac{\sigma^2}{\mu nT} )$.

\textbf{Instance 2. } Our proof for the term $\Omega({\mu \Delta}/{L}\exp(-\sqrt{\mu/L}\sqrt{1-\beta}T))$ builds on the similar idea to \cite{scaman2017optimal}: splitting the function used by Nesterov to prove the lower bound
for strongly convex and smooth optimization \cite{nesterov2003introductory}. However, the number of nodes $n$ is fixed in our analysis while \cite{scaman2017optimal} needs $n$ to be varying when establishing the lower bound.  Besides, our construction allows the objective functions to be suitable to an arbitrary initialization scope $\Delta =f(x^{(0)})-\min_{x}f(x)$ which, however, is not fully addressed and discussed in \cite{scaman2017optimal}.
These differences make our analysis novel and stronger.

We construct deterministic scenarios where no gradient noise is employed in the oracle. We assume the variable $x\in\ell_2\triangleq\{([x]_1,[x]_2,\dots,):\sum_{r=1}^{\infty}[x]_r^2<\infty\}$ to be infinitely dimensional and square-summable for simplicity. It is easy to adapt the argument for finitely dimensional variables as long as the dimension is proportionally larger than $T$. Without loss of generality, we assume all the algorithms start from $x^{(0)}=0\in\ell_2$. Let $M $ be
\begin{align*}
    M=\left[\begin{array}{ccccc}
2 & -1 & & &\\
-1 & 2 & -1 & &  \\
& -1 & 2 & -1 &   \\
& & \ddots & \ddots & \ddots 
\end{array}\right]\in\RR^{\infty\times \infty},
\end{align*}
then it is easy to see $0\preceq M\preceq 4I$. Let $E_1\triangleq\{j:1\leq j\leq \lceil \frac{n}{3}\rceil\}$ and $E_2\triangleq\{j:\lfloor \frac{n}{2}\rfloor +1\leq j\leq \lfloor\frac{n}{2}\rfloor+\lceil \frac{n}{3}\rceil\}$ as in Appendix \ref{app:lower-nonconvex}, and let
\begin{align}
    f_i(x)=\begin{cases}
    \frac{\mu}{2}\|x\|^2+\frac{L-\mu}{12}\frac{n}{\lceil n/3\rceil}\left([x]_1^2+\sum_{r\geq 1}([x]_{2r}-[x]_{2r+1})^2-2\lambda [x]_{1}\right)&\text{if }i\in E_1,\\
    \frac{\mu}{2}\|x\|^2+\frac{L-\mu}{12}\frac{n}{\lceil n/3\rceil}\sum_{r\geq 1}([x]_{2r-1}-[x]_{2r})^2&\text{if }i\in E_2,\\
    \frac{\mu}{2}\|x\|^2&\text{otherwise}.
    \end{cases}
\end{align}
where $\lambda\in\RR$ is to be specified.
It is easy to see that $[x]_1^2+\sum_{r\geq 1}([x]_{2r}-[x]_{2r+1})^2-2\lambda [x]_{1}$ and $\sum_{r\geq 1}([x]_{2r-1}-[x]_{2r})^2$ are convex and $4$-smooth. We thus have all $f_i$ are $L$-smooth and $\mu$-strongly convex, which implies  $f_i\in\cF_{L,\mu}$ for all $1\leq i\leq n$. We further have $f(x)=\frac{1}{n}\sum_{i=1}^nf_i(x)=\frac{\mu}{2}\|x\|^2+\frac{L-\mu}{12}\left(x^TMx-2\lambda [x]_1\right)$.

For the functions defined above, we establish that 
\begin{lemma}\label{lem:sc-communi}
Denote $\kappa :=L/\mu>1$, then it holds that for any $x$ and $r\geq 1$ satisfying $\prog(x)\leq r$,
\begin{align*}
f(x)-\min_{x}f(x) \geq \frac{\mu}{2L}\left(1-\frac{6}{\sqrt{3+6\kappa}+3}\right)^{2r}(f(x^{(0)})-\min_{x}f(x)).
\end{align*}
\end{lemma}
\begin{proof}
The minimum $x^\star$ of function $f$ satisfies $\left(\frac{L-\mu}{6}M+\mu\right)x-\lambda\frac{L-\mu}{6}e_1=0$, which is equivalent to 
\begin{align}
    \left(2+\frac{6}{\kappa-1}\right)[x]_1-[x]_2&=\lambda ,
    \nonumber \\
    -[x]_{j-1}+\left(2+\frac{6}{\kappa-1}\right)[x]_j-[x]_{j+1}&=0,\quad \forall\,j\geq 2.\label{eqn:vpwnmfwq}
\end{align}
Let $q$ be the smallest root of the equation $q^2-\left(2+\frac{6}{\kappa-1}\right)q+1=0$, then the variable 
\begin{align*}
    x^\star =\left([x^\star]_j=\lambda q^j\right)_{j\geq 1}
\end{align*}
satisfies \eqref{eqn:vpwnmfwq}. By the strong convexity of $f$, $x^\star$ is the unique solution. Therefore, when $\prog(x)\leq r$, it holds that 
\begin{align*}
    \|x-x^\star\|^2\geq \sum\limits_{j=r+1}^\infty \lambda^2q^{2j}=\lambda^2\frac{q^{2(r+1)}}{1-q^2}=q^{2r}\|x^{(0)}-x^\star\|^2.
\end{align*}
Finally, noting that 
\begin{align*}
    q=1+\frac{3-\sqrt{3+6\kappa}}{\kappa-1}=1-\frac{6}{\sqrt{3+6\kappa}+3},
\end{align*}
and  using the strong convexity of $f$, we reach the conclusion.
\end{proof}
Moreover, it is easy to see that the optimal value of $f(x)$ is
\begin{equation*}
    \min_{x}f(x)=f(x^\star)=-\frac{L-\mu}{12}\lambda [x^\star]_1=-\frac{L-\mu}{12}\lambda^2q.
\end{equation*}
Therefore, for any given $\Delta>0$, we can chose $\lambda=\sqrt{\frac{12 \Delta}{(L-\mu)q}}$ such that $f(x^{(0)})-\min_{x}f(x)=\Delta$.

Similar to the argument in Appendix \ref{app:lower-nonconvex}, we have
\begin{align*}
    \prog(\nabla f_i(x))
    \begin{cases}
    =\prog(x) +1&\text{if } \{\prog(x) \text{ is even and } i\in E_1\}\cup\{\prog(x) \text{ is odd and }i\in E_2\}\\
    \leq \prog(x) &\text{otherwise}.
    \end{cases}
\end{align*}
We further have 
\begin{equation}\label{eqn:jofqsfdq-fhisfaisf}
    \prog^{(T)}=\max_{1\leq i\leq n,\,0\leq s\leq T}\prog(x^{(t)}_i)\leq \left\lfloor \frac{T}{\dist(E_1,E_2)}\right\rfloor+1,\quad \forall\,T\geq 0.
\end{equation}
Therefore, with $T$ rounds of gossip communications budget in total, starting from $x_i^{(0)}=0$ for any $i=1,\dots,n$, any gossip algorithm $A$ can only achieves at most $\prog(\hat{x})=1+\lfloor \frac{T}{\dist(E_1,E_2)}\rfloor$, which, combined with Lemma \ref{lem:sc-communi}, leads to lower bound that 
\begin{align}\label{eqn:gvoqnfqfqfqt}
    &f(\hat{x})-\min_{x}f(x)\geq \frac{\mu \Delta}{2L}\left(1-\frac{6}{\sqrt{3+6\kappa}+3}\right)^{2+2\lfloor \frac{T}{\dist(E_1,E_2)}\rfloor}\nonumber\\
    =&\Omega\left(\frac{\mu}{L}\exp\left(-\frac{\sqrt{6}\lfloor \frac{T}{\dist(E_1,E_2)}\rfloor}{\sqrt{\kappa}}\right)\Delta\right).
\end{align}
where $\Delta$ measures the initialization $\mathbb{E}[f(x^{(0)})]-\min_{x}f(x)$.
The rest follows using the ring-lattice associated weight matrix $W\in\cW_{n,\beta}$  such that 
\begin{equation*}\label{eqn:w-constrc}
  \dist(E_1,E_2)=\Omega((1-\beta)^{-\frac{1}{2}}).
\end{equation*}

\subsection{Connectivity measures in common weight matrices}

Table \ref{table-order-beta} lists the order of the connectivity measure $\beta$ for weight matrices, if generated through the Laplacian rule $W = I - L/d_{\max}$ in which $L$ is the Laplacian matrix and $d_{\max}$ is the maximum degree, associated with commonly-used topologies. Since $\cos(\frac{\pi}{n}) \ge 1 - \frac{\pi^2}{2n^2}$, we have $[0, 1 - \frac{\pi^2}{2n^2}] \subseteq [0, \cos(\pi/n)]$. Since $\beta$ listed in Table \ref{table-order-beta} lies in $[0, 1 - \frac{\pi^2}{2n^2}]$ for sufficiently large $n$, we conclude that our lower bound established in Theorems \ref{thm-lower-bound-nc} and \ref{thm-lower-bound-pl} applies to weight matrices associated with topologies listed in Table \ref{table-order-beta}. 

\begin{table}[h!]
\centering 
\label{table-order-beta}
\caption{\small Order of connectivity measure $\beta$. ``E.-R. Rand'': Erdos-Renyi random graph $G(n,p)$ with probability $p=(1+a)\ln(n)/n$ for some $a>0$; ``Geo. Rand'': geometric random graph $G(n,r)$ with radius $r^2 = (1+a)\ln(n)/n$ for some $a > 0$.}
\begin{tabular}{ll}
\toprule
Topology & Order of $\beta$ \\ \midrule
Grid \cite{nedic2018network}     & 1-$\Theta([n\ln(n)]^{-1})$        \\
Torus \cite{nedic2018network}    & 1-$\Theta(n^{-1})$  \\
Hypercube \cite{nedic2018network}    & 1-$\Theta([\ln(n)]^{-1})$  \\
Exponential \cite{ying2021exponential}    & 1-$\Theta([\ln(n)]^{-1})$ \\
Complete     & 0  \\
E.-R. Rand.  \cite{benjamini2014mixing}   & 1-$\Theta([\ln(n)]^{-1})$  \\ 
Geo. Rand. \cite{boyd2005mixing}    & 1-$\Theta(\ln(n)n^{-1})$  \\\bottomrule
\end{tabular}
\vspace{-3mm}
\end{table}

\section{Convergence in MG-DSGD}
\label{app-mg-dsgd}

\subsection{Smooth and non-convex setting}
\label{app-mg-dsgd-smooth}

This section examines the convergence rate of \ours under the smooth and non-convex setting. Since \ours is a direct variant based on the vanilla DSGD, we first establish the convergence of DSGD to facilitate the convergence of \ours. 

\subsubsection{Convergence rate of DSGD}

The convergence rate of DSGD has been established in literatures such as \cite{koloskova2020unified,chen2021accelerating}. We adjust the analysis therein to achieve a slightly different result, which lays the foundation for the later convergence analysis in \ours.  

The following lemma established in \cite[Lemma~8]{koloskova2020unified} (also in \cite[Lemma~6]{chen2021accelerating}) shows how $\mathbb{E}[f(\bar{x}^{(k)})]$ evolves with iterations.
\begin{lemma}[\sc Descent Lemma \cite{koloskova2020unified}] \label{lm-descent}
	If $\{f_i(x)\}_{i=1}^n \subseteq \cF_{L}$, $\{\tilde{g}_i\} \subseteq O_{\sigma^2}$, and learning rate $\gamma < \frac{1}{4L}$, it holds for $k=0,1,2,\cdots$ that 
	\begin{align}\label{23bsd999}
	\mathbb{E}[f(\bar{x}^{(k+1)})] \le \mathbb{E}[f(\bar{x}^{(k)})] - \frac{\gamma}{4}\mathbb{E}\|\nabla f(\bar{x}^{(k)})\|^2 + \frac{3\gamma L^2}{4n}\mathbb{E}\|\vx^{(k)}- \bar{\vx}^{(k)}\|_F^2 + \frac{\gamma^2 \sigma^2 L}{2 n},
	\end{align}
	where $\bar{\vx}^{(k)} = [(\bar{x}^{(k)})^T; \cdots; (\bar{x}^{(k)})^T]\in \mathbb{R}^{n\times d}$, and $\bar{x}^{(k)} = \frac{1}{n}\sum_{i=1}^n x_i^{(k)}$. 
\end{lemma}

The following lemma established in \cite[Lemma~8]{chen2021accelerating} shows how the ergodic consensus term evolves with iterations. 

\begin{lemma}[\sc Consensus Lemma \cite{chen2021accelerating}]\label{lm-consensus-I} 
If $\{f_i(x)\}_{i=1}^n \subseteq \cF_{L}$, $\{\tilde{g}_i\} \subseteq O_{\sigma^2}$, $W \in \cW_{n,\beta}$, $\frac{1}{n}\sum_{i=1}^n\|\nabla f_i(x) - \nabla f(x)\|^2 \le b^2$ for any $x\in \RR^d$, and learning rate $\gamma \le \frac{1-\beta}{4\beta L}$, it holds that 
\begin{align}\label{23nbd}
\frac{1}{K+1}\sum_{k=0}^K \mathbb{E}\|\vx^{(k)} - \bar{\vx}^{(k)}\|_F^2 \le \frac{12 n \beta^2 \gamma^2 }{(1-\beta)^2(K+1)} \sum_{k =0}^K \mathbb{E}\|\nabla f(\bar{x}^{(k)})\|^2 + \frac{2 n \gamma^2 \beta^2}{1-\beta}(\frac{3b^2}{1-\beta} + \sigma^2). 
\end{align}
\end{lemma}
\begin{proof}
The Lemma 8 in \cite{chen2021accelerating} covers both DSGD and DSGD with periodic global average. To recover the result for DSGD, we let $C_\beta = 1/(1-\beta)$, $D_\beta = 1/(1-\beta)$ in \cite[Lemma~8]{chen2021accelerating} to achieve \eqref{23nbd}. 
\end{proof}

The following lemma establishes the convergence rate of DSGD under the smooth and non-convex setting. 
\begin{lemma}[\sc DSGD convergence]\label{thm:dsgd-noncovex}
If $\{f_i(x)\}_{i=1}^n \subseteq \cF_{L}$, $\{\tilde{g}_i\} \subseteq O_{\sigma^2}$, $W \in \cW_{n,\beta}$, $\frac{1}{n}\sum_{i=1}^n\|\nabla f_i(x) - \nabla f(x)\|^2 \le b^2$ for any $x\in \RR^d$, and learning rate is set as 
\begin{align}
\gamma = \min\{\frac{\sqrt{2 n \Delta}}{\sigma\sqrt{L(K+1)}}, \frac{1-\beta}{6\beta L}, \frac{1}{4L}\},
\end{align}
where $\Delta = \mathbb{E}[f(x^{(0)})] - f^\star$, it holds that 
\begin{align}\label{zb2309a9-2}
\frac{1}{K+1}\sum_{k=0}^K\mathbb{E}\|\nabla f(\bar{x}^{(k)})\|^2 = \cO\left( \frac{\sigma\sqrt{\Delta L}}{\sqrt{nK}} + \frac{\beta^2 L n \Delta}{(1-\beta) K} + \frac{\beta^2 L n b^2 \Delta }{(1-\beta)^2 K \sigma^2} + \frac{\beta L \Delta}{(1-\beta)K} + \frac{L\Delta}{K} \right) 
\end{align}
\end{lemma}

\begin{remark}
This result is slightly different from Theorem 2 in \cite{koloskova2020unified} where $\beta$ does not appear in the numerator of any terms in the upper bound therein. However, as we will show in the analysis of MG-DSGD, the $\beta$ appearing in the numerator of the third term in \eqref{zb2309a9-2} can significantly reduce the influence of data heterogeneity $b^2$ if $\beta \to 0$. 
\end{remark}

\begin{proof}
When $\gamma \le 1/(4L)$, we know from Lemma \ref{lm-descent} that 
\begin{align}
\mathbb{E}\|\nabla f(\bar{x}^{(k)})\|^2  \le \frac{4}{\gamma} \mathbb{E}[f(\bar{x}^{(k)})] - \frac{4}{\gamma} \mathbb{E}[f(\bar{x}^{(k+1)})]  + \frac{3 L^2}{n}\mathbb{E}\|\vx^{(k)}- \bar{\vx}^{(k)}\|_F^2 + \frac{2 \gamma \sigma^2 L}{n},
\end{align}
Taking running average over the above inequality, we obtain 
\begin{align}\label{zn28a56}
\frac{1}{K+1}\sum_{k=0}^K\mathbb{E}\|\nabla f(\bar{x}^{(k)})\|^2  &\le \frac{4}{\gamma(K+1)} \Big(\mathbb{E}[f(\bar{x}^{(0)})] - \mathbb{E}[f(\bar{x}^{(K+1)})]\Big)   + \frac{2 \gamma \sigma^2 L}{n}, \nonumber \\
&\quad\quad  + \frac{3 L^2}{n(K+1)}\sum_{k=0}^K \mathbb{E} \|\vx^{(k)}- \bar{\vx}^{(k)}\|_F^2
\end{align}
With Lemma \eqref{23nbd} and the fact that $\mathbb{E}[f(\bar{x}^{(0)})] - \mathbb{E}[f(\bar{x}^{(K+1)})] \le \mathbb{E}[f(\bar{x}^{(0)})] - f^\star$,  \eqref{zn28a56} becomes 
\begin{align}
&\ \Big( 1 - \frac{36\beta^2L^2\gamma^2}{(1-\beta)^2} \Big) \frac{1}{K+1}\sum_{k=0}^K\mathbb{E}\|\nabla f(\bar{x}^{(k)})\|^2 \nonumber \\
 \le&\ \frac{4\Big(\mathbb{E}[f(\bar{x}^{(0)})] - f^\star\Big)}{\gamma(K+1)}  + \frac{2 \gamma \sigma^2 L}{n} + \frac{6\gamma^2\beta^2L^2\sigma^2}{1-\beta} + \frac{18\gamma^2\beta^2L^2b^2}{(1-\beta)^2}
\end{align}
If learning rate $\gamma \le \frac{1-\beta}{12 \beta L}$, it holds that $1 - \frac{36\beta^2L^2\gamma^2}{1-\beta}  \ge 1/2$ and hence 
\begin{align}\label{zb2309a9}
\frac{1}{K+1}\sum_{k=0}^K\mathbb{E}\|\nabla f(\bar{x}^{(k)})\|^2 \le \frac{8\Delta}{\gamma(K+1)}  + \frac{4 \gamma \sigma^2 L}{n} + \frac{12\gamma^2\beta^2L^2\sigma^2}{1-\beta} + \frac{36\gamma^2\beta^2L^2b^2}{(1-\beta)^2}
\end{align}
where we introduce constant $\Delta = \mathbb{E}[f(\bar{x}^{(0)})] - f^\star$. Next we set 
\begin{align}\label{gammas}
\gamma_1 = \frac{\sqrt{2\Delta n}}{\sigma\sqrt{L(K+1)}}, \quad \gamma_2 = \frac{1-\beta}{6\beta L}, \quad \gamma_3 = \frac{1}{4L}, \quad \gamma = \min\{\gamma_1, \gamma_2, \gamma_3\}.
\end{align}
Apparently, it holds that 
\begin{align}\label{znbzxbzb2}
\frac{8\Delta}{\gamma(K+1)} &\le \frac{8\Delta}{\gamma_1 (K+1)} + \frac{8\Delta}{\gamma_2 (K+1)} + \frac{8\Delta}{\gamma_3 (K+1)} \nonumber \\
&= \cO\left(\frac{\sigma\sqrt{\Delta L}}{\sqrt{nK}} + \frac{\beta \Delta L}{(1-\beta)K} + \frac{\Delta L}{K}\right) 
\end{align}
and
\begin{align}\label{zbn0912655}
\frac{4 \gamma \sigma^2 L}{n} + \frac{12\gamma^2\beta^2L^2\sigma^2}{1-\beta} + \frac{36\gamma^2\beta^2L^2b^2}{(1-\beta)^2} \nonumber \le&\ \frac{4 \gamma_1 \sigma^2 L}{n} + \frac{12\gamma_1^2\beta^2L^2\sigma^2}{1-\beta} + \frac{36\gamma_1^2\beta^2L^2b^2}{(1-\beta)^2}  \nonumber \\
=&\ \cO\left(  \frac{\sigma\sqrt{\Delta L}}{\sqrt{nK}} + \frac{\beta^2 L \Delta n}{(1-\beta) K} + \frac{\beta^2 L \Delta n b^2}{(1-\beta)^2 K \sigma^2} \right)
\end{align}
Combining \eqref{zb2309a9}, \eqref{znbzxbzb2} and \eqref{zbn0912655}, we have \eqref{zb2309a9-2}. 

\end{proof}
\subsubsection{Convergence rate of \ours}\label{app:mg-dsgd-nonconvex}
Recall the recursions of DSGD and MG-DSGD as follows. 
\begin{align}
\mbox{(DSGD)} \quad \vx^{(k+1)}  &=  W \big(\vx^{(k)} - \gamma\nabla F(\vx^{(k)}; \xi^{(k)}) \big) \\
\mbox{(MG-DSGD)} \quad  \vx^{(k+1)}  &= 
\bar{M} \big(\vx^{(k)} - \gamma \vg^{(k)}) \big), \quad \mbox{where} \quad \vg^{(k)} = \frac{1}{R}\sum_{r=1}^R \nabla F(\vx^{(k)}; \xi^{(k,r)})
\end{align}
where $\bar{M}$ is doubly stochastic (see Prop. \ref{prop:fast-gossip}). MG-DSGD has two differences from the vanilla DSGD. First, the weight matrix $W$ is replaced with $\bar{M}$. Second, the stochastic gradient $\vg^{(k)}$ is achieved via gradient accumulation. This implies that the convergence analysis of \ours can follow that of vanilla DSGD. We only need to pay attentions to the
influence of $\bar{M}$ obtained by fast gossip averaging and the $\vg^{(k)}$ achieved by gradient accumulation. To proceed, we notice that 
\begin{equation*}
    \EE[\|g_i^{(k)}-\nabla f_i(x)\|^2]=\frac{1}{R}\EE[\|\nabla F(x_i^{(k)};\xi_i^{(k,r)})-\nabla f_i(x)\|^2]\leq \frac{\sigma^2}{R}
\end{equation*}
because $\{\xi^{(k,r)}_{i}\}_{r=1}^R$ are sampled independently. We introduce $\tilde{\sigma}^2\triangleq\sigma^2/R$ for notation simplicity. In addition, we know from Prop. \ref{prop:fast-gossip} that 
\begin{align*}
    \|\bar{M}-\frac{1}{n}\one_n\one_n^T\|_2\leq \tilde{\beta},\quad \mbox{where}\quad \tilde{\beta}:=\sqrt{2}\Big(1-\sqrt{1-\beta}\Big)^R.
\end{align*}
If we let 
\begin{equation}\label{eqn:gjowefmqwgweqx}
R=\left\lceil\frac{\max\{\ln(2),\frac{1}{2}\ln(n\max\{1,\frac{b^2}{\sigma^2\sqrt{1-\beta}}\})\}}{\sqrt{1-\beta}}\right\rceil=\tilde{O}\left(\frac{1}{\sqrt{1-\beta}}\right),
\end{equation}
then it holds that
\begin{equation}\label{eqn:hgiqrwqrwq}
    \tilde{\beta}\triangleq \sqrt{2}(1-\sqrt{1-\beta})^R\leq \sqrt{2}e^{-\sqrt{1-\beta}R}\leq \min\left\{\frac{1}{\sqrt{2}},\sqrt{2}\left(n\max\left\{1,\frac{b^2}{\sigma^2\sqrt{1-\beta}}\right\}\right)^{-\frac{1}{2}}\right\}.
\end{equation}
We thus immediately have $1-\tilde{\beta}=\Omega(1)$ and 
\begin{equation}\label{eqn:hgiqrwqrwq-daguda}
    \tilde{\beta}^2n\max\{1,{b^2R}/{\sigma^2}\}=\tilde{O}\left( \tilde{\beta}^2n\max\left\{1,\frac{b^2}{\sigma^2\sqrt{1-\beta}}\right\}\right)=\tilde{O}(1).
\end{equation}

\textbf{Theorem \ref{thm-MG-DSGD-rate-nc} (Formal version).} \textit{Given $L>0$, $n\geq 2$, $\beta \in[0,1)$, $\sigma>0$, and let $A$ denote Algorithm \ref{Algorithm: MG-DSGD}. Assuming that  $\frac{1}{n}\sum_{i=1}^n \|\nabla f_i(x) - \nabla f(x)\|^2 \le b^2$ for any $x\in \RR^{n}$ and the number of gossip rounds $R$ is set as in \eqref{eqn:gjowefmqwgweqx},
the convergence of $A$ can be bounded for any $\{f_i\}_{i=1}^n \subseteq \cF_{L}$ and any $W \in \cW_{n,\beta}$ that
\begin{equation}\label{MG-DSGD-rate-1-ap}
\setlength{\abovedisplayskip}{3pt}\setlength{\belowdisplayskip}{3pt}
\frac{1}{K}\sum_{k=1}^K \mathbb{E}\|\nabla f(\bar{x}^{(k)})\|^2 = \tilde{O}\left(\frac{{\sigma}\sqrt{\Delta L}}{\sqrt{nT}} + \frac{L\Delta}{\sqrt{1-\beta}T} \right) \quad \mbox{where} \quad \bar{x}^{(k)} = \frac{1}{n}\sum_{i=1}^n x_i^{(k)}
\end{equation}
}
and $\Delta = \mathbb{E}[f(\bar{x}^{(0)})] - f^\star$. 
\begin{remark}
Comparing with the established lower bound \eqref{eqn:fjowfqrgfsa}, we find the upper  bound \eqref{MG-DSGD-rate-1-ap} matches it tightly in all constants. 
\end{remark}
\begin{proof}
The convergence analysis of DSGD  applies to MG-DSGD by replacing $\sigma^2$ with $\tilde{\sigma}^2$, and $\beta$ with $\tilde{\beta}$. As a result, we know from Lemma \ref{thm:dsgd-noncovex} that if 
the learning rate is set as
\begin{align}
\gamma = \min\{\frac{\sqrt{\Delta n}}{\tilde{\sigma}\sqrt{L(K+1)}}, \frac{1-\tilde{\beta}}{6\tilde{\beta} L}, \frac{1}{4L}\}=\tilde{O}\left(\min\{\frac{\sqrt{2\Delta n}}{{\sigma}\sqrt{(1-\beta)LT}}, \frac{1}{L}\}\right)
\end{align}
where the last equality holds since $T = KR$ (where $T$ is the number of total gradient queries or communication rounds), it holds that 
\begin{align}
\frac{1}{K+1}\sum_{k=0}^K\mathbb{E}\|\nabla f(\bar{x}^{(k)})\|^2 &= \cO\left( \frac{\tilde{\sigma}\sqrt{\Delta L}}{\sqrt{nK}} + \frac{\tilde{\beta}^2 L \Delta n}{(1-\tilde{\beta}) K} + \frac{\tilde{\beta}^2 L \Delta n b^2}{(1-\tilde{\beta})^2 K \tilde{\sigma}^2} + \frac{\tilde{\beta} L \Delta}{(1-\tilde{\beta})K} + \frac{L\Delta}{K} \right) \nonumber\\
&\overset{\tilde{\beta}\leq \frac{1}{\sqrt{2}};~T=KR}{=} \cO\left( \frac{{\sigma}\sqrt{\Delta L}}{\sqrt{nT}} + \frac{\tilde{\beta}^2R L \Delta n}{ T} + \frac{\tilde{\beta}^2R^2 L \Delta n b^2}{ T{\sigma}^2}  + \frac{RL\Delta}{T} \right)\nonumber\\
&\overset{\eqref{eqn:gjowefmqwgweqx}}{=} \tilde{\cO}\left( \frac{{\sigma}\sqrt{\Delta L}}{\sqrt{nT}} + \frac{\tilde{\beta}^2R L \Delta n}{ T} + \frac{\tilde{\beta}^2R L \Delta n \frac{b^2}{\sigma^2\sqrt{1-\beta}}}{ T} +  \frac{RL\Delta}{T} \right)\nonumber\\
&\overset{\eqref{eqn:hgiqrwqrwq-daguda}}{=} \tilde{\cO}\left( \frac{{\sigma}\sqrt{\Delta L}}{\sqrt{nT}} + \frac{RL\Delta}{T} \right)\overset{\eqref{eqn:gjowefmqwgweqx}}{=} \tilde{\cO}\left( \frac{{\sigma}\sqrt{\Delta L}}{\sqrt{nT}} + \frac{L\Delta}{\sqrt{1-\beta}T} \right)\nonumber.
\end{align}
\end{proof}

\subsection{Smooth and non-convex setting under PL condition}
Similar to Appendix \ref{app-mg-dsgd-smooth}, we first establish the convergence of DSGD under the PL condition, and then derive the convergence of \ours. 

\subsubsection{Convergence rate of DSGD}
Substituting \eqref{PL-cond} to inequality \eqref{23bsd999}, we achieve the following descent lemma.
\begin{lemma}[\sc descent lemma]\label{lm-pl-descent}
If $\{f_i(x)\}_{i=1}^n \subseteq \cF_{L,\mu}$, $\{\tilde{g}_i\} \subseteq O_{\sigma^2}$ and learning rate $\gamma \le \frac{1}{4L}$, it holds for $k=0,1,2,\cdots$ that 
\begin{align}\label{23bsd999-1}
\mathbb{E}[f(\bar{x}^{(k+1)})] - f^\star \le (1 - \frac{\mu \gamma}{2})(\mathbb{E}[f(\bar{x}^{(k)})] - f^\star) + \frac{3\gamma L^2}{4n}\mathbb{E}\|\vx^{(k)}- \bar{\vx}^{(k)}\|_F^2 + \frac{\gamma^2 \sigma^2 L}{2 n}.
\end{align}
\end{lemma}

The following lemma establishes the consensus lemma for smooth and non-convex loss functions satisfying the PL condition.
\begin{lemma}[\sc Consensus lemma]\label{lm-pl-consensus}
	If $\{f_i(x)\}_{i=1}^n \subseteq \cF_{L,\mu}$, $\{\tilde{g}_i\} \subseteq O_{\sigma^2}$, $W \in \cW_{n,\beta}$, $\frac{1}{n}\sum_{i=1}^n\|\nabla f_i(x) - \nabla f(x)\|^2 \le b^2$ for any $x\in \RR^d$, and learning rate $\gamma \le \frac{1-\beta}{3\beta L}$, it holds for $k=0, 1,2,\cdots$ that 
	\begin{align}\label{znb23098anz6-2}
	\mathbb{E}\|{\vx}^{(k+1)}  - \bar{\vx}^{(k+1)}\|^2 \le&\ \big( \frac{1+\beta}{2} \big) \mathbb{E}\|{\vx}^{(k)}  - \bar{\vx}^{(k)}\|^2 + \frac{6n\beta^2 \gamma^2L}{1-\beta}\big(\mathbb{E}f(\bar{x}^{(k)}) - f^\star\big) \nonumber \\
	&\ + n\gamma^2\beta^2\sigma^2 + \frac{3n\beta^2 \gamma^2 b^2}{1-\beta}.
	\end{align}
\end{lemma}
\begin{proof}
When $\{f_i\}_{i=1}^L \subseteq \cF_L$, it is know from \cite[Eq.~(78)]{chen2021accelerating} that
\begin{align}\label{znb23098anz6}
\mathbb{E}\|{\vx}^{(k+1)}  - \bar{\vx}^{(k+1)}\|^2 \le&\ \big( \beta + \frac{3\beta^2\gamma^2 L^2}{1-\beta} \big) \mathbb{E}\|{\vx}^{(k)}  - \bar{\vx}^{(k)}\|^2 + \frac{3n\beta^2 \gamma^2 \mathbb{E} \|\nabla f(\bar{x}^{(k)})\|^2}{1-\beta} \nonumber \\
&\ + n\gamma^2\beta^2\sigma^2 + \frac{3n\beta^2 \gamma^2 b^2}{1-\beta}.
\end{align}
If $\gamma \le (1-\beta)/(3\beta L)$, it holds that 
\begin{align}\label{zb}
	\beta + \frac{3\beta^2\gamma^2 L^2}{1-\beta} \le \frac{1+\beta}{2}.
\end{align}
Moreover, since $\{f_i\}_{i=1}^L \subseteq \cF_L$ and hence $f(x) \in \cF_L$, we have 
\begin{align}\label{z623ba0}
\|\nabla f(\bar{x}^{(t)})\|^2 \le 2L\big(f(\bar{x}^{(t)}) - f^\star\big).
\end{align}
Substituting \eqref{zb} and \eqref{z623ba0} to \eqref{znb23098anz6}, we achieve \eqref{znb23098anz6-2}. 
\end{proof}

The following lemma establishes the convergence rate of DSGD under the PL condition. Using different proof techniques from those for the strongly-convex scenario in \cite{koloskova2020unified}, we can show how fast the last-iterate variable, i.e., $f(\bar{x}^{(K)}) - f^\star$, converges. In contrast, authors in \cite{koloskova2020unified} show the ergodic convergence, i.e., $\frac{1}{H_K}\sum_{i=1}^K h_k f(\bar{x}^{(k)}) - f^\star$ where $h_k >0$ is some positive weight and $H_K = \sum_{k=1}^K h_k$. Moreover, the $\beta$ appearing in the numerator of the third term in \eqref{pl-convg} can significantly reduce the influence of data heterogeneity $b^2$ when $\beta \to 0$. 

\begin{lemma}[\sc DSGD convergence]\label{thm:gjonfoqwfqw}
	If $\{f_i(x)\}_{i=1}^n \subseteq \cF_{L,\mu}$, $\{\tilde{g}_i\} \subseteq O_{\sigma^2}$, $W \in \cW_{n,\beta}$, $\frac{1}{n}\sum_{i=1}^n\|\nabla f_i(x) - \nabla f(x)\|^2 \le b^2$ for any $x\in \RR^d$, and learning rate $\gamma$ is set as 
	\begin{align}
	\gamma = \min\left\{\frac{4}{\mu K}\ln(\frac{\Delta \mu^2 n K}{\sigma^2 L}), \frac{1-\beta}{4L}, \frac{\mu(1-\beta)}{24n\beta L^2} \right\},
	\end{align}
	it holds that 
	\begin{align}
	&\mathbb{E}[f(\bar{x}^{(K)})] - f^\star  + L\mathbb{E}\|{\vx}^{(K)}  - \bar{\vx}^{(K)}\|^2 \label{pl-convg} \\
	=&\tilde{\cO}\Big(\frac{\sigma^2 L}{\mu^2 n K} + \frac{c_1 \beta^2 \sigma^2}{\mu^2 K^2 (1-\beta)} +  \frac{c_2  \beta^2 b^2}{\mu^2 K^2 (1-\beta)^2} + \Delta \exp(-\frac{\mu(1-\beta)K}{L}) + \Delta\exp(-\frac{\mu^2(1-\beta)K}{n \beta L^2}) \Big) \nonumber 
	\end{align}
\end{lemma}
where $\Delta=\mathbb{E}[f(x^{(0)})]-f^\star$, $c_1$ and $c_2$ are constants defined as follows
\begin{align}
  c_1=  \frac{6L}{\mu} + \frac{24 L^3}{\mu^2} + 4nL, \quad c_2=  \frac{18L^2}{\mu} + 12nL. 
\end{align}
\begin{proof}
With Lemmas \ref{lm-pl-descent} and \ref{lm-pl-consensus}, if $\gamma$ satisfies 
\begin{align}\label{gamma-cond1}
\gamma \le \frac{1}{4L}, \quad \mbox{and} \quad \gamma \le \frac{1-\beta}{3\beta L},
\end{align}
then we have 
\begin{align}\label{98z}
\underbrace{\left[
\begin{array}{l}
\mathbb{E}[f(\bar{x}^{(k+1)})] - f^\star   \vspace{1mm} \\
L\mathbb{E}\|{\vx}^{(k+1)}  - \bar{\vx}^{(k+1)}\|^2
\end{array}
\right]}_{:= z^{(k+1)}} &\ \le 
\underbrace{\left[
\begin{array}{cc}
1 - \frac{\mu\gamma}{2} & \frac{3\gamma L}{4n}  \vspace{1mm}\\ 
\frac{6n\beta^2\gamma^2L^2}{1-\beta} & \frac{1+\beta}{2} 
\end{array}
\right]}_{:=A}
\underbrace{\left[
\begin{array}{l}
\mathbb{E}[f(\bar{x}^{(k)})] - f^\star   \vspace{1mm} \\
L\mathbb{E}\|{\vx}^{(k)}  - \bar{\vx}^{(k)}\|^2
\end{array}
\right]}_{:= z^{(k)}}  \nonumber \\
&\quad + 
\underbrace{\left[
	\begin{array}{c}
	\frac{\gamma^2 \sigma^2 L}{2 n}  \vspace{1mm} \\
	nL\gamma^2\beta^2\sigma^2 + \frac{3nL\beta^2 \gamma^2 b^2}{1-\beta}
	\end{array}
	\right]}_{:= s} 
\end{align}
By recursing the above inequality, we have 
\begin{align}\label{z8zbzna021}
z^{(k)} \le A^k z^{(0)} + \sum_{\ell=0}^{k-1} A^{\ell} s \le A^k z^{(0)} + (I-A)^{-1}s
\end{align}
where the last inequality holds because $A$ is nonnegative and 
\begin{align}\label{z72bnabna}
\rho(A) \le \|A\|_1 \le \max\{1 - \frac{\mu\gamma}{2}  + \frac{6n\beta^2\gamma^2L^2}{1-\beta}, \frac{3\gamma L}{4n}  + \frac{1+\beta}{2} \} \le 1 - \frac{\mu\gamma}{4}. 
\end{align}
The last inequality in \eqref{z72bnabna} holds when  $\gamma$ satisfies
\begin{align}
	 \gamma \le \frac{(1-\beta)n}{3L}  &\quad \Longrightarrow \quad  \frac{3\gamma L}{4n}  + \frac{1+\beta}{2}  \le \frac{3+\beta}{4} \label{gamma-cond2}\\
	\gamma \le \frac{1-\beta}{\mu} &\quad \Longrightarrow \quad \frac{3+\beta}{4} \le 1 - \frac{\mu\gamma}{4}   \label{gamma-cond3}\\
	\gamma \le \frac{\mu(1-\beta)}{24n\beta^2 L^2} & \quad \Longrightarrow \quad  1 - \frac{\mu\gamma}{2}  + \frac{6n\beta^2\gamma^2L^2}{1-\beta} \le 1 - \frac{\mu\gamma}{4}  \label{gamma-cond4}
\end{align}
Furthermore, we can derive from \eqref{z8zbzna021} that 
\begin{align}\label{z8zbzna021-2}
\|z^{(k)}\|_1 &\le \|A^k\|_1 \|z^{(0)}\|_1 + \|(I-A)^{-1}s\|_1 \nonumber \\
&\le \|A\|_1^k \|z^{(0)}\|_1  + \|(I-A)^{-1}s\|_1 \nonumber \\
&\overset{\eqref{z72bnabna}}{\le} (1 - \frac{\mu\gamma}{4})^k \|z^{(0)}\|_1 + \|(I-A)^{-1}s\|_1. 
\end{align}
It is easy to verify that 
\begin{align}\label{znznzb23487}
	(I-A)^{-1} s &= \frac{1}{\frac{\mu\gamma(1-\beta)}{4} - \frac{9\beta^2\gamma^3L^3}{2(1-\beta)}}
	\left[
	\begin{array}{cc}
	\frac{1-\beta}{2} & \frac{3\gamma L}{4n} \vspace{1mm}\\
	\frac{6n\beta^2\gamma^2L^2}{1-\beta} & \frac{\mu \gamma}{2}
	\end{array}
	\right] \left[
	\begin{array}{c}
	\frac{\gamma^2 \sigma^2 L}{2 n}  \vspace{1mm} \\
	nL\gamma^2\beta^2\sigma^2 + \frac{3nL\beta^2 \gamma^2 b^2}{1-\beta}
	\end{array}
	\right] \nonumber \\
	&\overset{(a)}{\le} \frac{8}{\mu\gamma (1-\beta)} 	\left[
	\begin{array}{cc}
	\frac{1-\beta}{2} & \frac{3\gamma L}{4n} \vspace{1mm}\\
	\frac{6n\beta^2\gamma^2L^2}{1-\beta} & \frac{\mu \gamma}{2}
	\end{array}
	\right] \left[
	\begin{array}{c}
	\frac{\gamma^2 \sigma^2 L}{2 n}  \vspace{1mm} \\
	nL\gamma^2\beta^2\sigma^2 + \frac{3nL\beta^2 \gamma^2 b^2}{1-\beta}
	\end{array}
	\right] \nonumber \\
	&= 
	\left[
	\begin{array}{l}
	\frac{2 \gamma \sigma^2 L}{\mu n} + \frac{6 L^2 \gamma^2 \beta^2 \sigma^2}{\mu(1-\beta)} + \frac{18 L^2\beta^2 \gamma^2 b^2}{\mu(1-\beta)^2} \vspace{1mm}\\
	\frac{24\beta^2\gamma^3\sigma^2L^3}{\mu(1-\beta)^2} + \frac{4nL\gamma^2\beta^2\sigma^2}{1-\beta} + \frac{12nL\beta^2\gamma^2b^2}{(1-\beta)^2}
	\end{array}
	\right],
\end{align}
where (a) holds when 
\begin{align}
\gamma \le \frac{1-\beta}{6\beta L}\sqrt{\frac{\mu}{L}}. \label{gamma-cond5}
\end{align}
Inequality \eqref{znznzb23487} implies that 
\begin{align}\label{zzbb90267az}
\|(I-A)^{-1} s\|_{1} &\le \frac{2 \gamma \sigma^2 L}{\mu n}  + \big( \frac{6L^2}{\mu} + \frac{24\gamma L^3}{\mu(1-\beta)} + 4nL \big) \frac{\gamma^2 \beta^2 \sigma^2}{1-\beta} + \big( \frac{18L^2}{\mu} + 12nL\big)\frac{\gamma^2 \beta^2 b^2}{(1-\beta)^2} \nonumber \\
&\overset{(a)}{\le} \frac{2\gamma \sigma^2 L}{\mu n}  + \big( \frac{6L}{\mu} + \frac{24 L^3}{\mu^2} + 4nL \big) \frac{\gamma^2 \beta^2 \sigma^2}{1-\beta} + \big( \frac{18L^2}{\mu} + 12nL\big)\frac{\gamma^2 \beta^2 b^2}{(1-\beta)^2} \nonumber \\
&\overset{(b)}{=} \frac{2 \gamma \sigma^2 L}{\mu n}  +  \frac{c_1 \gamma^2 \beta^2 \sigma^2}{1-\beta} +  \frac{c_2 \gamma^2 \beta^2 b^2}{(1-\beta)^2} 
\end{align}
where (a) holds because $\gamma$ satisfies \eqref{gamma-cond3}, and (b) holds by introducing 
\begin{align}
c_1:=  \frac{6L}{\mu} + \frac{24 L^3}{\mu^2} + 4nL, \quad c_2:=  \frac{18L^2}{\mu} + 12nL. 
\end{align}
Substituting \eqref{zzbb90267az} to \eqref{z8zbzna021-2}, and recalling the definition of $z^{(k)}$ in \eqref{98z}, we have (at iteration $K$)
\begin{align}\label{z091}
&\ \mathbb{E}[f(\bar{x}^{(K)})] - f^\star  + L\mathbb{E}\|{\vx}^{(K)}  - \bar{\vx}^{(K)}\|^2 \nonumber \\
\le&\  \Delta (1 - \frac{\mu\gamma}{4})^K + \frac{2\gamma \sigma^2 L}{\mu n}  +  \frac{c_1 \gamma^2 \beta^2 \sigma^2}{1-\beta} +  \frac{c_2 \gamma^2 \beta^2 b^2}{(1-\beta)^2} \nonumber \\
\le&\ \Delta \exp(-\frac{\mu\gamma K}{4}) + \frac{2 \gamma \sigma^2 L}{\mu n}  +  \frac{c_1 \gamma^2 \beta^2 \sigma^2}{1-\beta} +  \frac{c_2 \gamma^2 \beta^2 b^2}{(1-\beta)^2}
\end{align}
where $\Delta =\mathbb{E}[f(\bar{x}^{(0)})] - f^\star  + L\mathbb{E}\|{\vx}^{(0)}  - \bar{\vx}^{(0)}\|^2 = \mathbb{E}[f({x}^{(0)})] - f^\star$ if $x_i^{(0)} = 0$ for any $i \in [n]$. Next we let 
\begin{align}
	\gamma_1 = \frac{4}{\mu K}\ln(\frac{\Delta \mu^2 n K}{\sigma^2 L}), \quad \gamma_2 = \frac{1-\beta}{4L}, \quad \gamma_3 = \frac{\mu(1-\beta)}{24n\beta L^2}, \quad \gamma = \min\{\gamma_1, \gamma_2, \gamma_3\}.  \label{gamma-conds}
\end{align}
It is easy to verify that $\gamma$ satisfies all conditions in \eqref{gamma-cond1}, \eqref{gamma-cond2}, \eqref{gamma-cond3}, \eqref{gamma-cond4} and \eqref{gamma-cond5}. With \eqref{gamma-conds}, we have 
\begin{align}\label{27aba091-1}
\Delta \exp(-\frac{\mu\gamma K}{4}) &\le \Delta \big[  \exp(-\frac{\mu\gamma_1 K}{4}) + \exp(-\frac{\mu\gamma_2 K}{4})+ \exp(-\frac{\mu\gamma_3 K}{4})\big] \nonumber \\
&= \cO\Big(\frac{\sigma^2 L}{\mu^2 n K} +\Delta \exp(-\frac{\mu(1-\beta)K}{L}) + \Delta\exp(-\frac{\mu^2(1-\beta)K}{n \beta L^2}) \Big)
\end{align}
and 
\begin{align}\label{27aba091-2}
 \frac{\gamma \sigma^2 L}{\mu n}  +  \frac{c_1 \gamma^2 \beta^2 \sigma^2}{1-\beta} +  \frac{c_2 \gamma^2 \beta^2 b^2}{(1-\beta)^2} \le &\  \frac{\gamma_1 \sigma^2 L}{\mu n}  +  \frac{c_1 \gamma_1^2 \beta^2 \sigma^2}{1-\beta} +  \frac{c_2 \gamma_1^2 \beta^2 b^2}{(1-\beta)^2}  \nonumber \\
= &\ \tilde{\cO}\Big(\frac{\sigma^2 L}{\mu^2 n K} + \frac{c_1 \beta^2 \sigma^2}{\mu^2 K^2 (1-\beta)} +  \frac{c_2  \beta^2 b^2}{\mu^2 K^2 (1-\beta)^2} \Big)
\end{align}
Combining \eqref{z091}, \eqref{27aba091-1} and \eqref{27aba091-2}, we achieve \eqref{pl-convg}.  

\end{proof}

\subsubsection{Convergence rate of \ours}
As discussed in Appendix \ref{app:mg-dsgd-nonconvex}, it holds that $\tilde{\sigma}^2\triangleq\sigma^2/R$ and $\tilde{\beta}\triangleq \sqrt{2}(1-\sqrt{1-\beta})^R$. If we let 
\begin{equation}\label{eqn:gjowefmqwgweqx-pl}
R=\left\lceil\frac{\max\{\ln(2),\ln(\frac{nL}{\mu}),\frac{1}{2}\ln(\frac{n}{L}\max\{c_1,\frac{c_2b^2}{\sigma^2\sqrt{1-\beta}}\})\}}{\sqrt{1-\beta}}\right\rceil=\tilde{O}\left(\frac{1}{\sqrt{1-\beta}}\right),
\end{equation}
we then have
\begin{equation}\label{eqn:hgiqrwqrwq-hfisadas-pl}
    \tilde{\beta}\triangleq \sqrt{2}(1-\sqrt{1-\beta})^R\leq \sqrt{2}e^{-\sqrt{1-\beta}R}\leq \min\left\{\frac{1}{\sqrt{2}},\sqrt{2}\frac{\mu}{nL}, \left(2\frac{L}{n}\min\left\{\frac{1}{c_1},\frac{\sigma^2\sqrt{1-\beta}}{b^2c_2}\right\}\right)^{\frac{1}{2}}\right\}.
\end{equation}
which immediately leads to $1-\tilde{\beta}=\Omega(1)$ and 
\begin{equation}\label{eqn:hgiqrwqrwq-daguda-pl}
    \tilde{\beta}^2\max\{c_1,c_2R\frac{{b^2}}{\sigma^2}\}=\tilde{O}\left( \tilde{\beta}^2\max\left\{c_1,c_2\frac{{b^2}}{\sqrt{1-\beta}\sigma^2}\right\}\right)=\tilde{O}(\frac{L}{n}).
\end{equation}

\textbf{Theorem \ref{thm-MG-DSGD-rate-pl} (Formal version).} \textit{Under the same assumptions as in Theorem \ref{thm-MG-DSGD-rate-nc} and setting $R$ as in \eqref{eqn:gjowefmqwgweqx-pl}, the convergence of $A$ can be bounded for any loss functions $\{f_i\}_{i=1}^n \subseteq \cF_{L,\mu}$, and any $W \in \cW_{n,\beta}$ by
	\begin{align}\label{MG-DSGD-rate-2-app}
	\setlength{\abovedisplayskip}{3pt}\setlength{\belowdisplayskip}{3pt}
	 \mathbb{E}[f(\bar{x}^{(K)}) - f^\star] = \tilde{\cO}\left(  \frac{{\sigma}^2 L}{\mu^2 n T} + \Delta \exp(-\frac{\mu \sqrt{1-\beta} T }{L}) \right).
	\end{align}
where $\Delta = \mathbb{E}[f({x}^{(0)})] - f^\star$.  
}
\begin{remark}
Comparing with the lower bound established in \eqref{eqn:fjowfqrgfsa-fasfasfa}, we find the upper bound \eqref{MG-DSGD-rate-2-app} matches it tightly in terms of $T$, $n$, $\sigma^2$, and $\beta$, but is a little bit worse in dependence on $L$ and $\mu$. We will leave it as a future work to develop new algorithms that can match \eqref{eqn:fjowfqrgfsa-fasfasfa} tightly in all constants. 
\end{remark}

\begin{proof}
The convergence analysis of DSGD  applies to MG-DSGD by replacing $\sigma^2$ with $\tilde{\sigma}^2$, and $\beta$ with $\tilde{\beta}$. As a result, we know from Lemma \ref{thm:gjonfoqwfqw} that if 
the learning rate is set as
\begin{align}
\gamma = \min\left\{\frac{4}{\mu K}\ln(\frac{\Delta \mu^2 n K}{\tilde{\sigma}^2 L}), \frac{1-\tilde{\beta}}{4L}, \frac{\mu(1-\tilde{\beta})}{24n\tilde{\beta} L^2} \right\}=\tilde{O}\left(\min\{\frac{1}{\mu\sqrt{1-\beta}T}, \frac{1}{L}\}\right)
\end{align}
it holds that 
\begin{align}
&\frac{1}{K+1}\sum_{k=0}^K\mathbb{E}\|\nabla f(\bar{x}^{(k)})\|^2 \nonumber\\
=& \tilde{\cO}\Big(\frac{\tilde{\sigma}^2 L}{\mu^2 n K} + \frac{c_1 \tilde{\beta}^2 \tilde{\sigma}^2}{\mu^2 K^2 (1-\tilde{\beta})} +  \frac{c_2  \tilde{\beta}^2 b^2}{\mu^2 K^2 (1-\tilde{\beta})^2} + \Delta \exp(-\frac{\mu(1-\tilde{\beta})K}{L}) + \Delta\exp(-\frac{\mu^2(1-\tilde{\beta})K}{n \tilde{\beta} L^2}) \Big) \nonumber\\
&\overset{\tilde{\beta}\leq \frac{1}{\sqrt{2}};~T=KR\geq R}{=} \cO\left( \frac{{\sigma}^2 L}{\mu^2 n T}+ \frac{c_1 \tilde{\beta}^2 {\sigma}^2}{\mu^2 T} +  \frac{c_2  \tilde{\beta}^2R b^2}{\mu^2 T } + \Delta \exp(-\frac{\mu T}{RL}) + \Delta\exp(-\frac{\mu^2T}{n \tilde{\beta}R L^2}) \right)\nonumber\\
&\overset{\eqref{eqn:hgiqrwqrwq-daguda-pl},\eqref{eqn:hgiqrwqrwq-hfisadas-pl}}{=} \tilde{\cO}\left(  \frac{{\sigma}^2 L}{\mu^2 n T} + \Delta \exp(-\frac{\mu T}{RL}) \right)\overset{\eqref{eqn:gjowefmqwgweqx-pl}}{=} \tilde{\cO}\left(  \frac{{\sigma}^2 L}{\mu^2 n T} + \Delta \exp(-\sqrt{1-\beta}\frac{\mu T}{L}) \right)\nonumber.
\end{align}
\end{proof}

\section{Experiments}
\label{app-exp}

\subsection{Performance in terms of epochs}
Fig.~\ref{fig:cifar10_alpha1} depicts the performance of several algorithms in terms of the communicated messages. In Fig.~\ref{fig:cifar10_alpha1_appendix} we illustrate their performances in terms of epochs. It is observed that \ours achieves slightly better validation accuracy than other baselines in both CIFAR-10 and ImageNet dataset.

\begin{figure}[!h]
    \centering
    \includegraphics[width=0.4\textwidth]{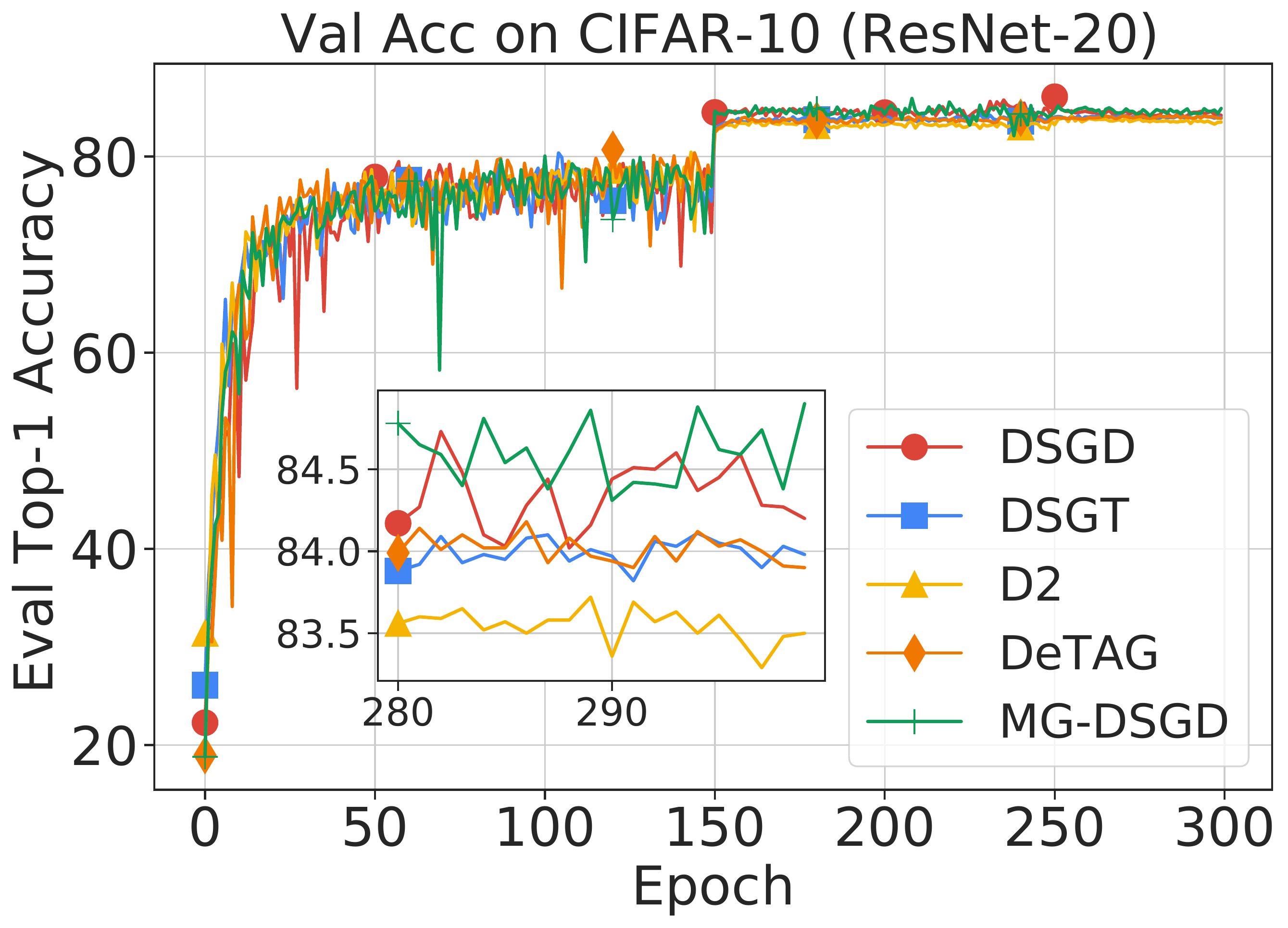}
    \includegraphics[width=0.4\textwidth]{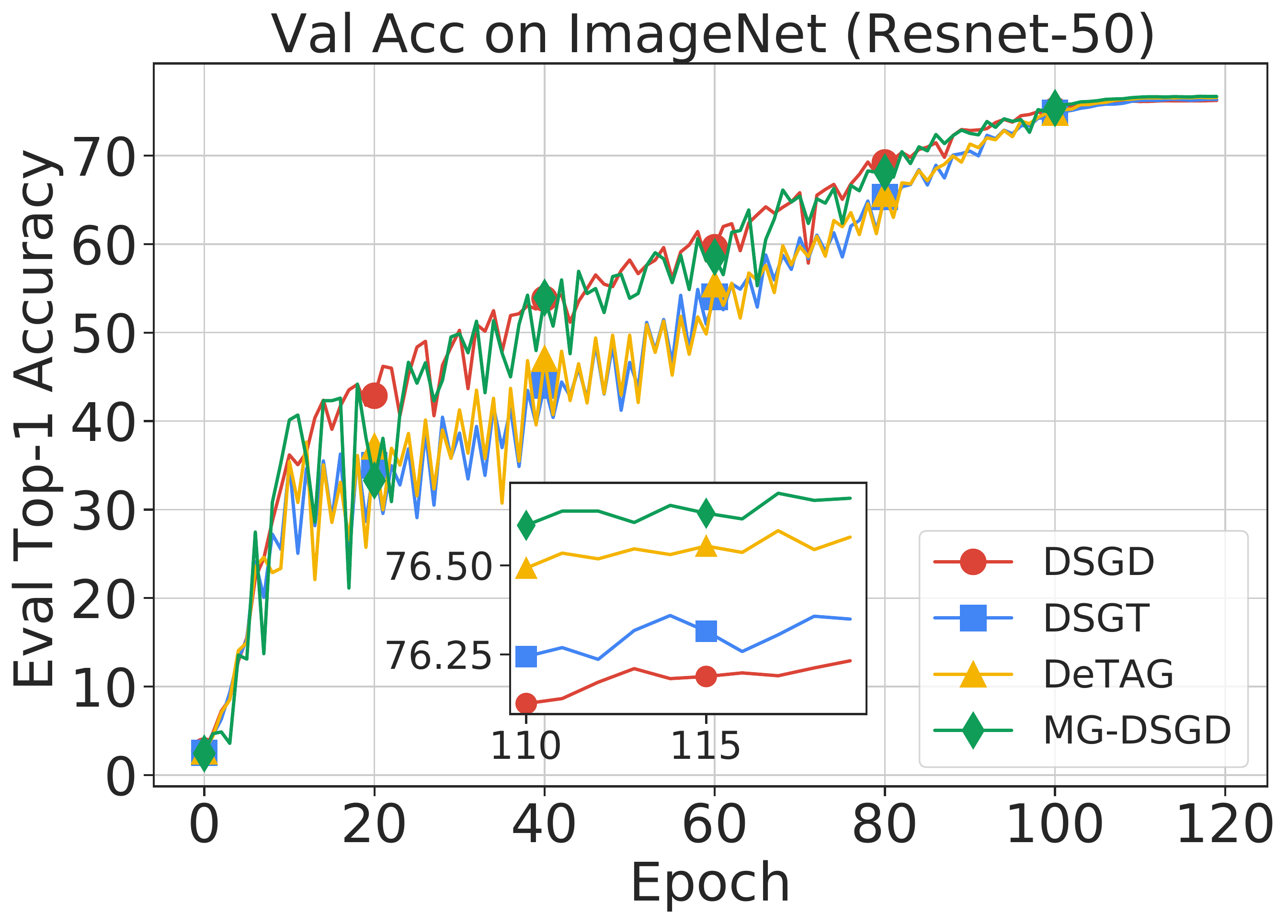}
    \vskip -3mm
    \caption{Convergence results in terms of validation accuracy. Left: CIFAR-10 with heterogeneous data ($\alpha=1$); Right: ImageNet with homogeneous data.}
    \label{fig:cifar10_alpha1_appendix}
\end{figure}

\subsection{Effects of accumulation rounds}
We further investigate the performance of the proposed \ours and \detag over different accumulation rounds. The ``M-'' prefix indicates methods with momentum acceleration. As shown in Table~\ref{table-accu-rounds}, \ours reaches almost the same performance as \detag with less communications. Both methods have performance degradation as accumulation round scales up. We conjecture that while gradient accumulation, which amounts to using large-batch samples in gradient evaluation, can help in the optimization and training stage, it may hurt the generalization performance because  gradient with less variance can lead the algorithm to a sharp local minimum. As a result, we recommend using \ours in applications that are friendly to large-batch training. 

\begin{table}[h!]
\centering

\caption{Effects of round numbers for \cifar dataset}
\begin{small}
\begin{sc}
\setlength{\tabcolsep}{0.95mm}{
\begin{tabular}{ccccccccc}

\toprule
Models &  \multicolumn{4}{c}{ResNet18} &  \multicolumn{4}{c}{ ResNet20 } \\

Rounds  & 2 & 3 & 4 & 5  & 2 & 3 & 4 & 5   \\
\midrule
m-DeTAG & 94.40±.13 & 93.50±.75 & 92.88±.50 & 92.57±.66 & 91.77±.25 & 91.50±.14 & 91.19±.26 & 90.62±.31 \\
m-Ours & 94.57±.05 & 93.95±.11 & 93.15±.25 & 92.01±.91 & 91.77±.09  & 91.17±.07 & 91.19±.13 & 90.44±.36 \\ 
DeTAG & 93.17±.18 & 92.59±.03 & 92.26±.23 & 91.48±.27 & 88.97±.11 & 89.02±.08 & 88.77±.09 & 88.37±.24 \\ 
Ours & 93.75±.12  & 92.72±.20 & 92.43±.32 & 91.78±.28 & 89.18±.08 & 88.92±.10 & 88.80±.15 & 88.52±.20 \\ 

\bottomrule
\end{tabular}}
\end{sc}
\end{small}
\label{table-accu-rounds}
\end{table}

\end{document}